\documentclass[journal,onecolumn]{IEEEtran}
\IEEEoverridecommandlockouts                             
\usepackage{amsmath,graphicx,multicol}
\usepackage{amsfonts}
\usepackage{color}
\usepackage{bbm}
\usepackage{cite}
\usepackage{amssymb}
\usepackage{algorithm}
\usepackage{algorithmic,booktabs}
\usepackage{hyperref}
\usepackage{tabularx}
\usepackage[utf8]{inputenc}
\usepackage[english]{babel}
\usepackage{dsfont}
\usepackage{subcaption}
\usepackage[font={footnotesize}]{caption}
\usepackage{url}
\usepackage{mathtools}
\usepackage{comment}
\usepackage{tcolorbox}
\usepackage{float}
\usepackage{stfloats}
\usepackage{amssymb}
\usepackage{fancyhdr}
\usepackage{cleveref}
\usepackage{listings}
\usepackage{xcolor}

\usepackage{amsthm}
\usepackage{soul}

\newtheorem{assumption}[]{Assumption}
\newtheorem{definition}{Definition}[]

\newtheorem{theorem}{Theorem}[]
\newtheorem{corollary}{Corollary}[]
\newtheorem{lemma}[]{Lemma}
\theoremstyle{remark}

\DeclareMathOperator*{\argmin}{arg\,min}

\DeclareMathOperator{\E}{\mathbb{E}}

\usepackage{amsmath}
\makeatletter
\newcommand{\mathleft}{\@fleqntrue\@mathmargin0pt}
\newcommand{\mathcenter}{\@fleqnfalse}
\makeatother

\begin{document}
\title{\LARGE \textbf{
 FedNMUT---Federated Noisy Model Update Tracking Convergence Analysis}}
\author{
Vishnu Pandi Chellapandi, Antesh Upadhyay, Abolfazl Hashemi, and Stanislaw H. \.{Z}ak
\thanks{V. P. Chellapandi, A. Upadhyay, A. Hashemi, and S. H. \.{Z}ak are with Elmore Family School of Electrical and Computer Engineering, Purdue University, West Lafayette, 47907, USA. Emails: {\tt\small \{cvp,aantesh,abolfazl,zak\}@purdue.edu}}
}

\maketitle
\begin{abstract}
A novel Decentralized Noisy Model Update Tracking Federated Learning algorithm (FedNMUT) is proposed that is tailored to function efficiently in the presence of noisy communication channels that reflect imperfect information exchange. This algorithm uses gradient tracking to minimize the impact of data heterogeneity while minimizing communication overhead. The proposed algorithm incorporates noise into its parameters to mimic the conditions of noisy communication channels, thereby enabling consensus among clients through a communication graph topology in such challenging environments. FedNMUT prioritizes parameter sharing and noise incorporation to increase the resilience of decentralized learning systems against noisy communications. Theoretical results for the smooth non-convex objective function are provided by us, and it is shown that the $\epsilon-$stationary solution is achieved by our algorithm at the rate of $\mathcal{O}\left(\frac{1}{\sqrt{T}}\right)$, where $T$ is the total number of communication rounds. Additionally, via empirical validation, we demonstrated that the performance of FedNMUT is superior to the existing state-of-the-art methods and conventional parameter-mixing approaches in dealing with imperfect information sharing. This proves the capability of the proposed algorithm to counteract the negative effects of communication noise in a decentralized learning framework. 
\end{abstract}

\section{Introduction}
\subsection{Motivation}
In the real world, vast amounts of data are generated from various sources such as computers, mobile devices, smartwatches, and vehicles. These data are aggregated in centralized data centers for the purpose of training machine learning models. However, this centralized approach encounters significant hurdles, including limited communication bandwidth and privacy concerns, rendering it unreliable and non-scalable. Consequently, there is a pressing need to establish learning paradigms that not only ensure data security but also uphold privacy standards. This necessity has spurred the development of decentralized optimization algorithms, exemplified by Decentralized Federated Learning (DFL) and Federated Learning (FL). These methodologies have been increasingly applied across diverse sectors, including smart cities and connected vehicles \cite{nedic2009distributed,koloskova2020unified,hashemi2021benefits,mcmahan2017communication,konevcny2016federated,chellapandi2023survey,pandya2023federated,chellapandi2023federated,yuan2023fedmfs}. In this paper, we propose a novel Decentralized Noisy Model Update Tracking Federated Learning algorithm (FedNMUT) that is tailored to function efficiently in the presence of noisy communication channels that reflect imperfect information exchange.

\subsection{Literature Overview}
Decentralized optimization typically employs consensus-based gradient descent methods, where the computed parameters are shared among clients \cite{nedic2018network,nedic2009distributed,tsitsiklis1984problems,chong2023introduction}. These clients independently calculate local weights and gradients based on their data, subsequently exchanging these parameters with others. The aggregation of these parameters, influenced by the network's topology which governs the communication framework, is pivotal in this learning paradigm. This network topology is often depicted as a simple graph, illustrating the communication pathways among clients. Decentralized approaches aim to alleviate the limitations inherent in centralized systems, such as communication delays and bandwidth constraints, thereby enhancing scalability and efficiency in extensive networks.

Recent advancements like Gradient Tracking (GT) and Momentum Tracking (MT) have been devised to tackle the challenge of data heterogeneity in decentralized settings, albeit at the cost of increased communication overhead \cite{di2016next,pu2021distributed,lin2021quasi,takezawa2022momentum}. Conversely, the Quasi-Global Momentum (QGM) strategy achieves global momentum synchronization without necessitating additional communication, thus mitigating decentralized learning challenges related to diverse data sets. Furthermore, RelaySGD, a novel method, replaces traditional gossip averaging with Relay-Sum, offering a unique approach to the averaging process. The integration of RelaySGD with existing methods can significantly improve performance. The QG-DSGDm algorithm, which combines QGM with DSGDm, marks a notable advancement in this domain. A recent proposition, the Global Update Tracking (GUT) method, addresses data heterogeneity in decentralized learning efficiently, without incurring extra communication costs. Empirical results validate that GUT not only accommodates variances in data distribution across devices but also bolsters the performance of decentralized learning processes \cite{aketi2024global}.

While enhancing communication efficiency remains a crucial challenge in distributed learning, efforts like communication compression have been explored \cite{du2020high,zheng2020design,hashemi2021benefits,chen2021communication,chen2021decentralized,li2022detection}. However, these initiatives often presume ideal, noise-free communication channels. The robustness and dependability of machine learning frameworks, especially in the burgeoning fields reliant on distributed learning, hinge on their performance amidst noisy communications.

Imperfect information exchange, such as noisy or quantized communication, has been examined in the context of average consensus algorithms within distributed frameworks. Yet, the ramifications of varying noise levels remain underexplored. Moreover, existing research, primarily focused on consensus issues, does not fully address the complex challenges encountered in contemporary decentralized optimization and learning paradigms \cite{carli2007average,qin2021communication}. In contrast to Federated Learning (FL), where server assistance is common, Decentralized Federated Learning (DFL) operates without a central server, with each client acting autonomously, processing local Stochastic Gradient Descent (SGD) or its variants on its data and interacting directly with neighboring clients.

In our previous paper~\cite{chellapandi2023convergence}, we performed a comparative study of three proposed algorithms for DFL under imperfect communication conditions, typified by noisy channels. These algorithms—FedNDL1, FedNDL2, and FedNDL3—differ in their handling of noise and parameter sharing, demonstrating varying degrees of resilience to communication noise. In this paper, we propose a novel algorithm that employs the Gradient Tracking method in DFL and compare its performance against the previously mentioned algorithms.

\subsection{Paper's Contributions}

This paper introduces a novel algorithm that employs the Gradient Tracking method in DFL, considering the impact of communication noise. Previous studies have evaluated the effectiveness of two-time scale methods in DFL with noisy channels. However, these investigations were limited by inflexible assumptions such as strong convexity in papers such as \cite{reisizadeh2019exact, vasconcelos2021improved, reisizadeh2022distributed, reisizadeh2022almost}. These assumptions are rarely satisfied in practical and large-scale learning scenarios, which limits the applicability of the proposed methods. The new algorithm presented in this paper addresses these limitations by using a more flexible optimization framework that can handle smooth non-convex objective functions. We conducted experiments using the proposed algorithm in practical distributed learning scenarios, under the presence of noise. The results have shown that the algorithm is capable of reducing overall loss and mitigating consensus error, despite the presence of noise.

We evaluate the performance of decentralized FL using the Federated Model Update Tracking algorithm, which incorporates noise into tracking parameter transmissions to clients or servers. Unlike prior models where communication noise was not considered, our approach integrates noise post-local SGD updates, facilitating parameter exchange via a gossip/mixing matrix and updating the global parameters thereafter. We also compare the noise resilience of our algorithm with previously developed algorithms, providing both theoretical and empirical evidence of its robustness against communication noise~\cite{chellapandi2023convergence}.We show that the algorithm proposed in this paper handles the communication noise better than our previously proposed FedNDL3 algorithm.

\section{Problem Statement, Algorithm, and Assumptions}
\subsection{Problem Statement}
In this section, we define the framework, the underlying assumptions, and the methodologies evaluated in this study. The discussion begins with a typical DFL configuration, wherein \(n\) clients possess distinct local datasets and perform consensus-based learning to obtain the global parameters. The problem is mathematically formulated as follows:
\begin{equation}
    \min_{x \in \mathbb{R}^{d}}\left[f(x) = \frac{1}{n}\sum_{i=1}^{n} f_i(x_i) \right],
\end{equation}
where each \(f_i:\mathbb{R}^{d} \rightarrow \mathbb{R}\), for \(i\) ranging from 1 to \(n\), signifies the local objective function for the \(i^{th}\) client node. The stochastic expression of the local objective function is presented as:
\begin{equation}
    f_i(x_i) = \mathbb{E}_{\xi_i \sim \mathcal{D}_{i}}[F_i(x_i,\xi_i)],
\end{equation}
in which \(\xi_i\) represents the sample data drawn from the data distribution \(\mathcal{D}_{i}\) specific to the \(i^{th}\) client. Here, \(F_i(x_i,\xi_i)\) denotes the loss function calculated for each client and their respective data sample \(\xi_i\). The notation \(x_i \in \mathbb{R}^d\) refers to the parameter vector for client \(i\), while \(X \in \mathbb{R}^{d\times n}\) is the matrix constituted by these parameter vectors. The fundamental goal for the clients is to collaboratively attain a state of optimality, denoted as \(x_i = x^{*} = \argmin_{x \in \mathbb{R}^d} f(x)\), which represents the global minimum.

\begin{definition}[\textbf{Mixing matrix}]
    The mixing matrix, $W = [w_{ij}] \in [0,1]^{n\times n}$, is a non-negative, symmetric, that is, $W = W^{\top}$, and doubly stochastic matrix, that is, $W\mathds{1} = \mathds{1}, \mathds{1}^{\top}W = \mathds{1}^{\top}$, where $\mathds{1}$ is the column vector of size $n$ whose elements are ones.
\end{definition}

\textbf{\textit{FedNDL1~\cite{chellapandi2023convergence}: }}In FedNDL1, the model updates are performed by each client in parallel and then each client then communicates the parameters to the neighbors. The communication is topology dependent. The neighboring client receives a noisy version of the parameters due to the imperfect communication channel,

\begin{equation}
    \label{eq:consensus-step}
    x_i^{(t+1)} \ = \ \sum_{j=1}^n w_{ij} \ ( x_j^{(t+\frac{1}{2})} + \delta_j^{(t)} ),
\end{equation}
where $x_j^{(t+\frac{1}{2})}$ is the vector of parameters sent by client $j$ and $\delta_j^{(t)} \in \mathbb{R}^d$, is a zero mean random noise. We assume the noise to have a zero mean, the noise variance is
\begin{equation*}
    \label{eq:noise-var}
    D^2_{t,j} = \E[\|\delta_j^{(t)}\|^2].
\end{equation*}

\textit{\textbf{FedNDL2~\cite{chellapandi2023convergence}:} }
Similar to the previous algorithm, this algorithm also performs a two-stage process. However, In FedNDL2, the consensus step is performed 
before computing the individual gradients and parameters,
\begin{equation}
    \label{eq:p2-gossip}
    x_i^{(t+\frac{1}{2})} =  \sum_{j=1}^{n} w_{ij} ( x_j^{(t)} + \delta_j^{(t)} ),
\end{equation}

\textbf{\textit{FedNDL3~\cite{chellapandi2023convergence}:}} In FedNDL3, 
 the clients share their gradients over a noisy communication channel instead of the weights followed by the SGD update. This idea comes from our Noisy-FL motivation and the fact that SGD is inherently a noisy process. So, pursuing this scenario gives more flexibility to handle the noise as a part of the SGD process. The entire formulation for this algorithm can be written as, 
\begin{equation}
    \label{eq:p3-gradient-based}
    x_i^{(t+1)} \ = \ x_i^{(t)} - \eta_t \sum_{j=1}^{n} w_{ij} \ ( g_j^{(t)} + \delta_j^{(t)} ),
\end{equation}
where $g_j^{(t)}$ refers to the gradient of client $j$ at iteration $t$

\captionsetup{font=normal}
\captionsetup[sub]{font=normal}
\begin{figure*}[!t]
\vspace{0.1in}
\begin{subfigure}{0.33\textwidth}
    \centering    \includegraphics[width=\textwidth]{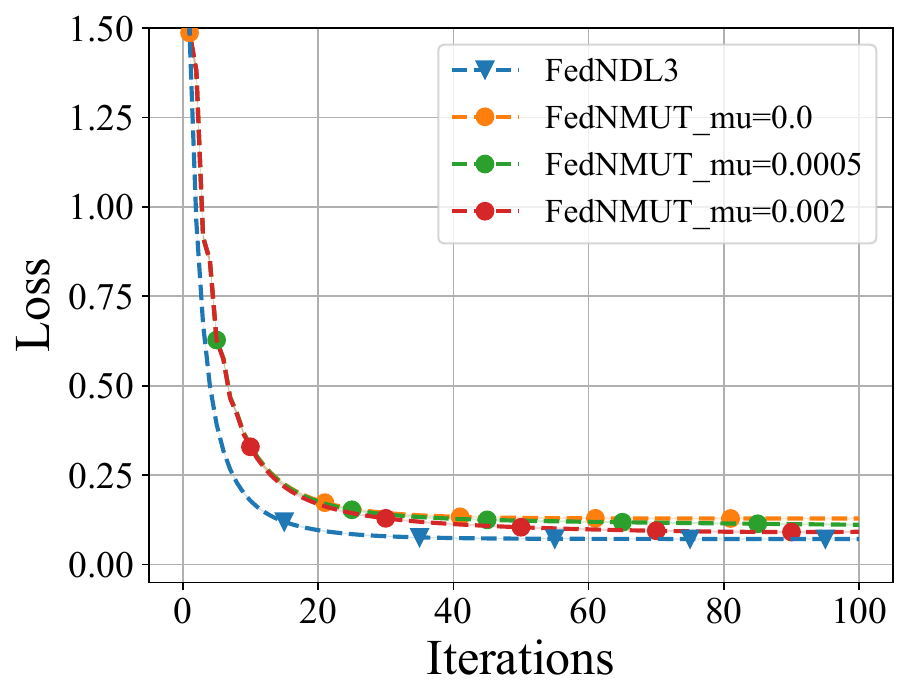}
    \caption{Fully-connected topology}
    \label{fig:syn_1_0}
\end{subfigure}
\begin{subfigure}{0.33\textwidth}
    \centering
    \includegraphics[width=\textwidth]{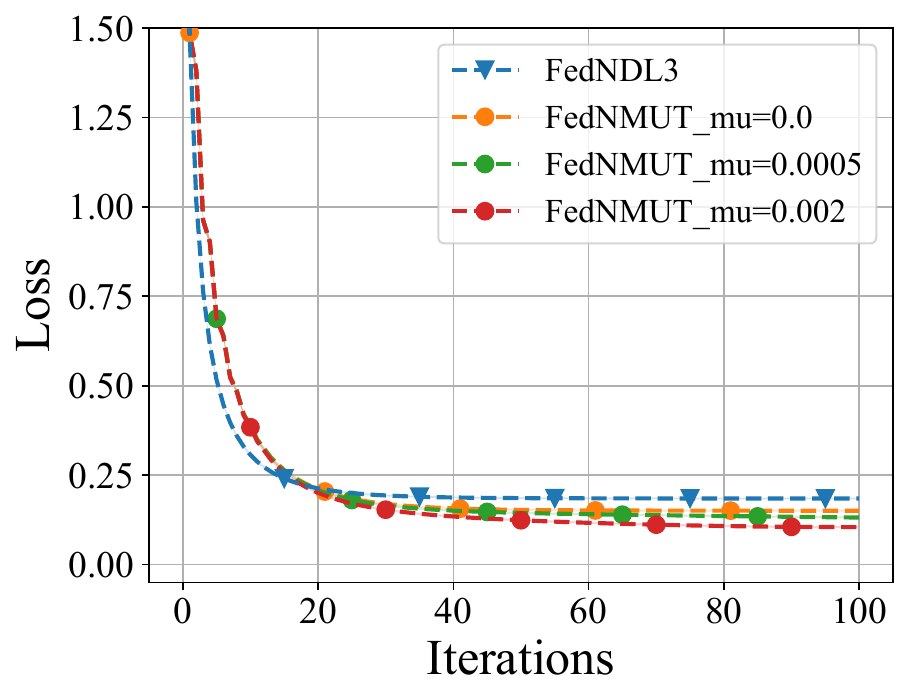}
    \caption{Torus topology}
    \label{fig:syn_1_0.005}
\end{subfigure}
\begin{subfigure}{0.33\textwidth}
    \centering
    \includegraphics[width=\textwidth]{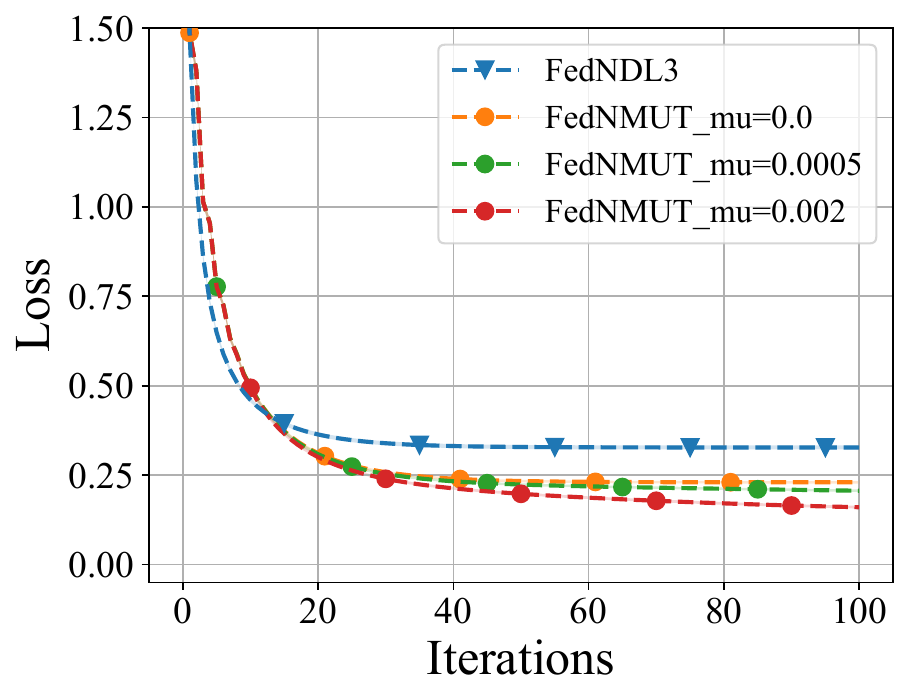}
    \caption{Ring topology} 
    \label{fig:syn_1_0.001}
\end{subfigure} 
\caption{Loss versus iterations for various $\mu$ values for different communication topologies (No noise scenario).}
\label{fig:loss_all_mu}
\vspace{-4mm}
\end{figure*}

\begin{figure*}[!t]
\begin{subfigure}{0.33\textwidth}
    \centering
    \includegraphics[width=\textwidth]{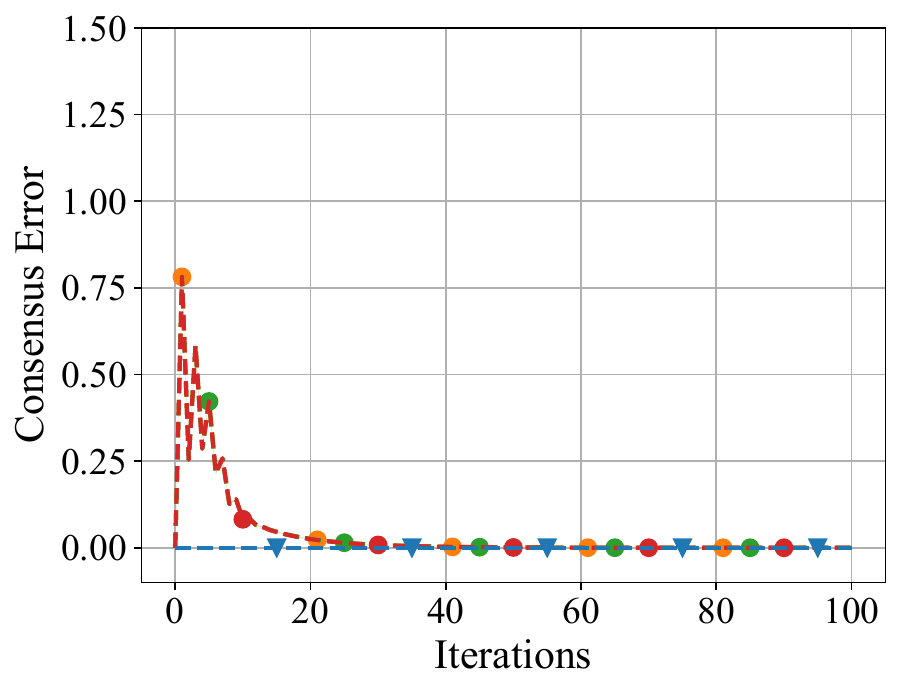}
    \caption{Fully-connected topology} 
    \label{fig:syn_1_0.005_full_cons}
\end{subfigure}
\begin{subfigure}{0.33\textwidth}
    \centering
    \includegraphics[width=\textwidth]{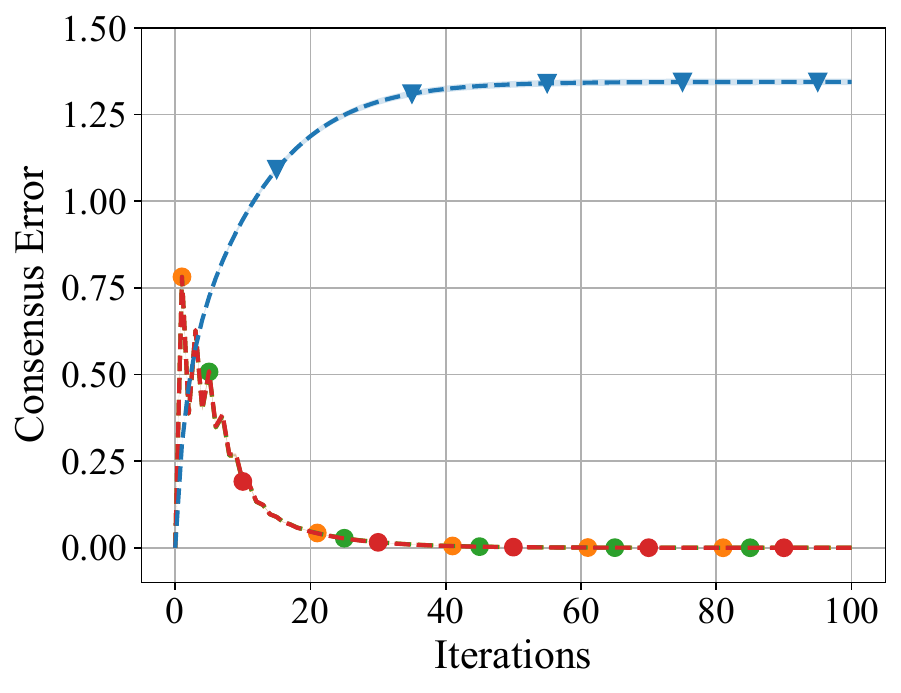}
    \caption{Torus topology}
    \label{fig:syn_1_0.005_torus_cons}
\end{subfigure}
\begin{subfigure}{0.33\textwidth}
    \centering
    \includegraphics[width=\textwidth]{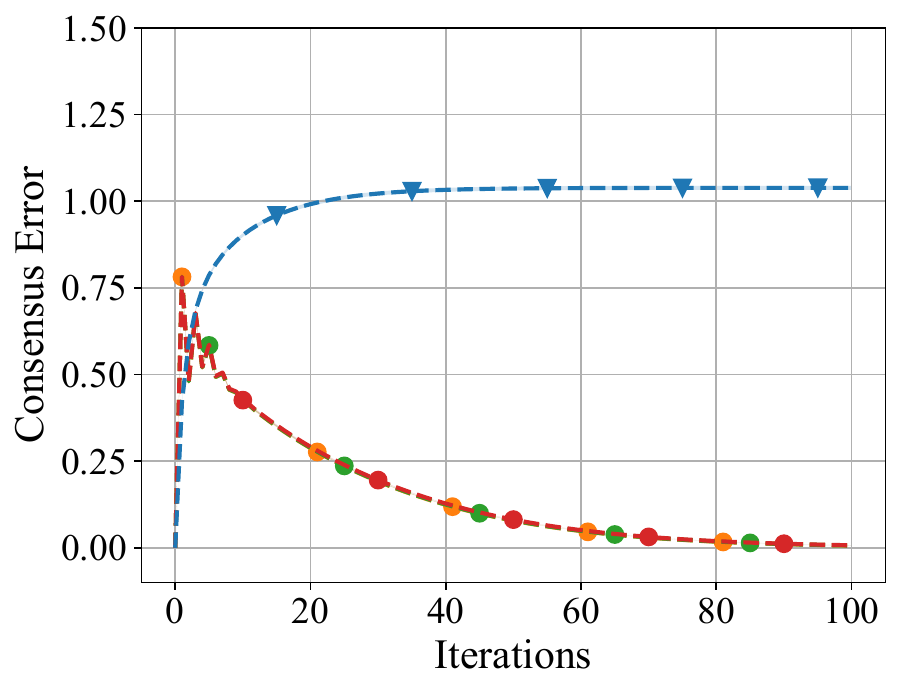}
    \caption{Ring topology} 
    \label{fig:syn_1_0.005_ring_cons}
\end{subfigure}
\caption{Consensus error versus iterationsfor various $\mu$ values for different communication topologies (No noise scenario).}
\label{fig:consensus_all_mu}
\vspace{-4mm}
\end{figure*}

\textbf{FedNMUT:} We next describe our proposed algorithm studied in this paper. The global update tracking algorithm aims to gain the advantages of gradient tracking while eliminating communication overhead. Instead of exchanging gradients, clients communicate model updates to their neighbors.

In the Global Update Tracking (GUT) approach, the model updates, $x_i^t-x_i^{t-1}$, are transferred instead of the gradients $g_i^t$. Thus, this approach involves each client $i$ transmitting its model updates to its neighbors, thereby avoiding the direct communication of model gradients. Each client maintains a record of its neighbor's model states as $\hat{x}_j$, updating this information with incoming model updates to maintain the latest state information, as outlined in line~10 of the algorithm.

In this algorithm, the variable $\Delta_i^t$, is calculated for each client $i$, which aggregates the local gradient update $g_i^t$ and the gossip averaging update $\sum_j (w_{ij}-I_{ij})\hat{x}_j^t$, as shown in line~5 of the Algorithm~\ref{alg:FedNMUT}. The tracking variable $y_i^t$ is calculated as per line~6 in the algorithm, factoring in both the local gradient and the gossip averaging updates represented by $\Delta_i^t$. The gossip component of the update for each client $i$, expressed as $ \sum_j w_{ij}(\hat{x}_j^t-x_i^t)$, is based on the client's own model parameters. To integrate this in the calculation of the tracking variable $y_i^t$, the information from neighboring clients ($y_j^t$) must be adjusted to reflect the client's own reference frame, resulting in an added term $\frac{1}{\eta}(\hat{x}_j^t-x_i^t)$ in the update formula mentioned in line~6 of the algorithm. This term can be consolidated and viewed as an bias term, $b^t$ as shown in equation~\eqref{eq:bias_term} below. Additionally, to optimize the effectiveness of this method, the correction factor for the tracking variable is scaled by $\mu$, a hyperparameter that can be adjusted to maximize the algorithm's performance. 

The term denoted $\delta^{(t)}_i$ represents the communication noise that is added to the tracking variable $y_i^t$ shown in line~7 of the algorithm. The noisy tracking variable is then transmitted to neighboring clients as shown in lines~8--10 of the algorithm. Setting the hyperparameter $\mu$ allows us to replicate the FedNDL3 algorithm update as per equation~\eqref{eq:p3-gradient-based}.

We summarize the above described algorithms in a compact way as Algorithm~\cref{alg:all-noisy-DFL}.

\definecolor{light-gray}{gray}{0.85}
\begin{algorithm}[h]
    \caption{\textcolor{darkgray}{FedNDL1},  \textcolor{blue}{FedNDL2}, and \textcolor{purple}{FedNDL3}}
    \label{alg:all-noisy-DFL}
    \begin{algorithmic}[1]
       \STATE {\bfseries Input:} For each node $i$ initialize: $x_i^{(0)} \  \in \ \mathbb{R}^d$, step size $ \{ \eta_t \}_{t=0}^{T-1} $, mixing matrix $W$, noise from the communication channel $\delta^{(t)}$\\
       \FOR{$t = 0, \dots, T$}
       \STATE \textcolor{darkgray}{ \textbf{FedNDL1:} } 
       \STATE \textcolor{darkgray}{ {Run in parallel for each client $i$}}
       \STATE  \textcolor{darkgray}{ Sample $\xi_{i}^{(t)} \mbox{, compute } g_i^{(t)} = \widetilde{\nabla} f_i (x_i^{(t)},\xi_i^{(t)})$}
       \STATE  \textcolor{darkgray}{ $ x_i^{(t+\frac{1}{2})} \ = \  x_i^{(t)} - \eta_t g_i^{(t)} $\ }
       \STATE \textcolor{darkgray}{ $ x_i^{(t+1)} \ = \ \sum_{j=1}^{n} w_{ij} \ ( x_j^{(t+\frac{1}{2})} + \delta_j^{(t)} ) $\ }              
       \STATE \textcolor{blue}{ \textbf{FedNDL2:} } 
      \STATE \textcolor{blue}{ $ x_i^{(t+\frac{1}{2})} \ = \ \sum_{j=1}^{n} w_{ij} \ ( x_j^{(t)} + \delta_j^{(t)} ) \  $ }
       \STATE \textcolor{blue}{ {Run in parallel for each clients $i$}}
       \STATE \textcolor{blue}{ Sample $\xi_{i}^{(t)} \mbox{ ,  } g_i^{(t+\frac{1}{2})}= \widetilde{\nabla} f_i (x_i^{(t+\frac{1}{2})},\xi_i^{(t)})$ \ }
       \STATE \textcolor{blue}{ $ x_i^{(t+1)} \ = \  x_i^{(t+\frac{1}{2})} - \eta_t g_i^{(t+\frac{1}{2})}  $ }

       \STATE \textcolor{purple}{ \textbf{FedNDL3:} }  
       \STATE \textcolor{purple}{ Run in parallel for each client $i$ }
       \STATE \textcolor{purple}{ Sample $\xi_{i}^{(t)} \mbox{, compute } g_i^{(t)} = \widetilde{\nabla} f_i (x_i^{(t)},\xi_i^{(t)})$}
       \STATE \textcolor{purple}{ $ x_i^{(t+1)} \ = \ x_i^{(t)} - \eta_t \sum_{j=1}^{n} w_{ij} \ ( g_j^{(t)} + \delta_j^{(t)} ) \ $ }
       \ENDFOR
    \end{algorithmic}
\end{algorithm}

The FedNMUT algorithm can be formulated as

\begin{equation}
    \label{eq:FedNMUT-based}
x_i^{(t+1)} = x_i^{(t)} - \eta_t\widetilde{ y_i}^t  = x_i^{(t)} - \eta_t ({ y_i}^t + \delta_i^{t}),  \\
\end{equation}
where

\begin{equation}    
\label{eq:bias_term}
y_i^{t}= \Delta_i^t + \mu \underbrace{\Big[\sum\limits_{j\in \mathcal{N}(i)}w_{ij}(\widetilde{y_j}^{t-1}-\frac{1}{\eta}(\hat{x}_j^{t}-x_i^{t}))  - \Delta_i^{t-1}\Big]}_{b^t},\\
\end{equation}
and
\begin{equation}
\Delta_i^{t}= g_{i}^{t} - \frac{1}{\eta}\sum_{j\in \mathcal{N}(i)} w_{ij}(\hat{x}_j^t-x_i^t).
\nonumber
\end{equation}

\begin{algorithm}[h]
    \caption{FedNMUT - Noisy Model Update Tracking }
    \begin{algorithmic}[1]
       \STATE {\bfseries Input:} For each node $i$ initialize: $x_i^{(0)} \  \in \ \mathbb{R}^d$, step size $ \{ \eta_t \}_{t=0}^{T-1} $, neighbors' copy $\hat{x}_j^{0}$, step size $\eta$, scaling factor $\mu$, mixing matrix $W=[w_{ij}]_{i,j \in [1,n]}$, $\mathcal{N}(i)$ represents neighbors of $i$ including itself, noise from the communication channel $\delta^{(t)}$\\
       \FOR{$t = 0, \dots, T - 1 $}
       \STATE { {Run in parallel for each client $i$}}
       \STATE { Sample $\xi_{i}^{(t)} \mbox{, compute } g_i^{(t)} = \widetilde{\nabla} F_i (x_i^{(t)},\xi_i^{(t)})$}
       \STATE{ $\Delta_i^{t}= g_{i}^{t} - \frac{1}{\eta}\sum_{j\in \mathcal{N}(i)} w_{ij}(\hat{x}_j^t-x_i^t)$}
       \STATE{$y_i^{t}= \Delta_i^t + \mu \Big[\sum\limits_{j\in \mathcal{N}(i)}w_{ij}(\widetilde{y_j}^{t-1}-\frac{1}{\eta}(\hat{x}_j^{t}-x_i^{t}))  - \Delta_i^{t-1}\Big]$}
       \STATE{$\widetilde{y_i}^{t}= y_i^{t} + \delta_i^{(t)} $}       \STATE{S\text{\scriptsize END}R\text{\scriptsize ECEIVE}($ \widetilde{y_i}^t$)}
       \STATE{$x_i^{t+1}= x_i^{t} - \eta \widetilde{y_i}^t$ }
       \STATE{$\hat{x}_j^{t+1}=\hat{x}_j^{t}-\eta \widetilde{y_j}^{t} \hspace{2mm} \forall \hspace{2mm} j \in \mathcal{N}(i) \backslash i$}          
       \ENDFOR
    \end{algorithmic}
    \label{alg:FedNMUT}
\end{algorithm}

\subsection{Assumptions}
The assumptions used in the convergence analysis of the proposed FedNMUT decentralized algorithm are commonly used in the literature, see~\cite{hashemi2021benefits,koloskova2020unified,koloskova2019decentralized}.
\begin{assumption}[\textbf{Smoothness}]
\label{as1} 
The objective function $F_i(x, \xi)$ is $L$-smooth with respect to $x$, for all $\xi$. Each $f_i(x)$ is $L$-smooth, that is,
\begin{equation}
\|\nabla f_i (x) - \nabla f_i (y)\| \leq L \|x-y\|, \qquad \mbox{for all x,y} .
\nonumber
\end{equation}
Hence the function $f$ is also $L$-smooth.
\end{assumption}
\begin{assumption}[\textbf{Bounded Variance}] 
\label{as2} 
The variance of the stochastic gradient of each client $i$ is bounded, that is, $$\E[||{\nabla}F_i (x_i^t,\xi_i^t) - \nabla f_i (x_i^t)||^2] \leq \sigma^2,$$ where $\xi_i^t$ denotes random batch of samples in client node $i$ for $t^{th}$ round, and ${\nabla}F_i(x_i^t,\xi_i^t)$ denotes the stochastic gradient. In addition,  we also assume that the stochastic gradient is unbiased, that is, $\E[{\nabla}F_i (x_i^t,\xi_i^t)] = \nabla f_i (x_i^t)$.
\end{assumption}
\begin{assumption}[\textbf{Mixing matrix}] 
\label{as3} 
 For $\rho \in (0,1]$, the mixing matrix W satisfies,
\begin{equation*}
    \|(\Bar{X} - X)W \|_{F}^2 \leq (1 - \rho)\|\Bar{X} - X\|_{F}^2,
\end{equation*}
which means that the gossip averaging step brings the columns of $X \in \mathbb{R}^{d \times n}$ closer to the row-wise average, that is, $\Bar{X}= X\frac{\mathds{1}\mathds{1}^{\top}}{n}$.
\end{assumption}
{Note that standard topologies such as ring,  torus, and	fully-connected satisfy the above assumption.} 
\begin{assumption}[\textbf{Noise model}] 
\label{as4} 
The noise present due to contamination of communication channel $\delta_i^{(t)}$ is independent, has zero mean and bounded variance, that is, \\
\hspace{10mm}$\E[\delta_i^{(t)}] = 0$ and $\E[||\delta_i^{(t)}||^2] = D_{t,i}^2 < \infty$. 
\end{assumption}
The above assumption is specific to the imperfect information sharing setup and is also used in \cite{reisizadeh2022distributed,reisizadeh2022almost,upadhyay2023improved,wei2022federated}.

\begin{assumption}[\textbf{Bounded Client Dissimilarity (BCD)}] 
\label{as5} 
For all $x \in \mathbb{R}^{d}$, 
\begin{equation*}
    \label{eq:bcd}
    \frac{1}{n}\sum_{i=1}^{n}\|\nabla f_{i}(x) - \nabla f(x)\|^2 \leq \zeta^2
\end{equation*}
for some constant $\zeta$.
\end{assumption}
The above assumption is made to limit the extent of client heterogeneity and is standard in the DFL setup.



\captionsetup{font=normal}
\captionsetup[sub]{font=normal}
\begin{figure*}[t]
\vspace{0.1in}
\begin{subfigure}{0.33\textwidth}
    \centering    \includegraphics[width=\textwidth]{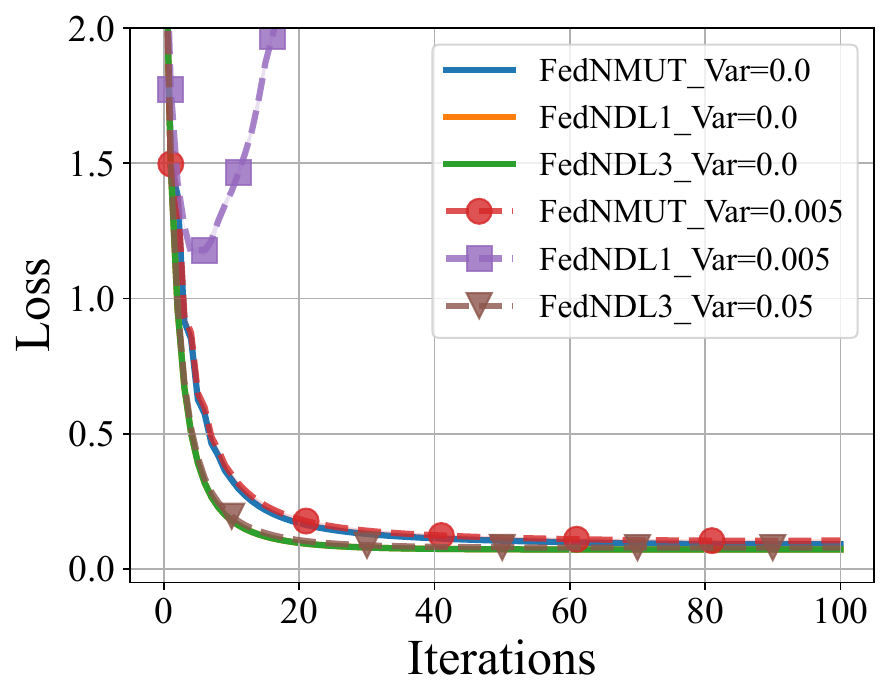}
    \caption{Fully-connected topology}
    \label{fig:syn_1_0}
\end{subfigure}
\begin{subfigure}{0.33\textwidth}
    \centering
    \includegraphics[width=\textwidth]{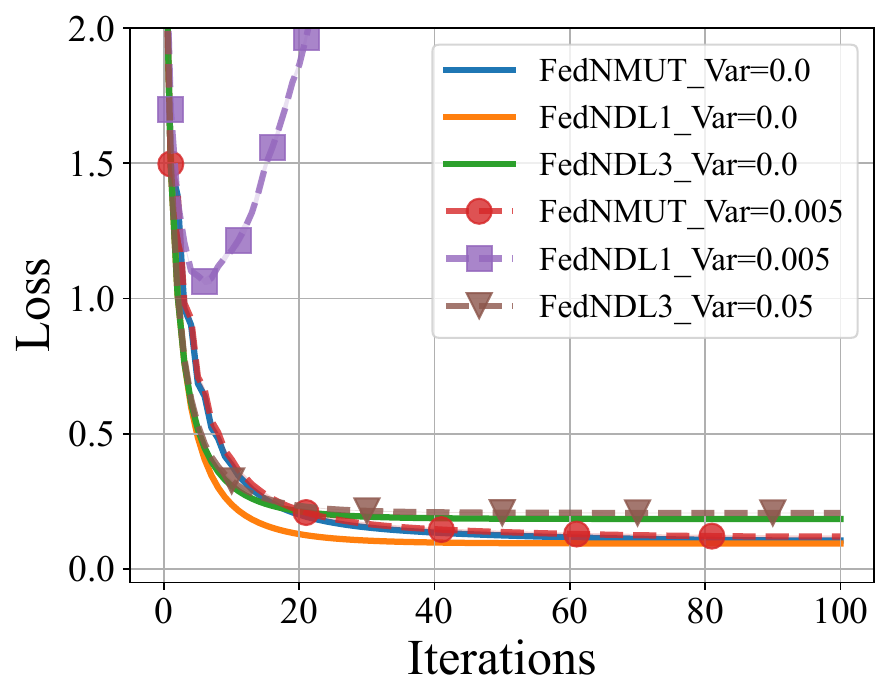}
    \caption{Torus topology}
    \label{fig:syn_1_0.005}
\end{subfigure}
\begin{subfigure}{0.33\textwidth}
    \centering
    \includegraphics[width=\textwidth]{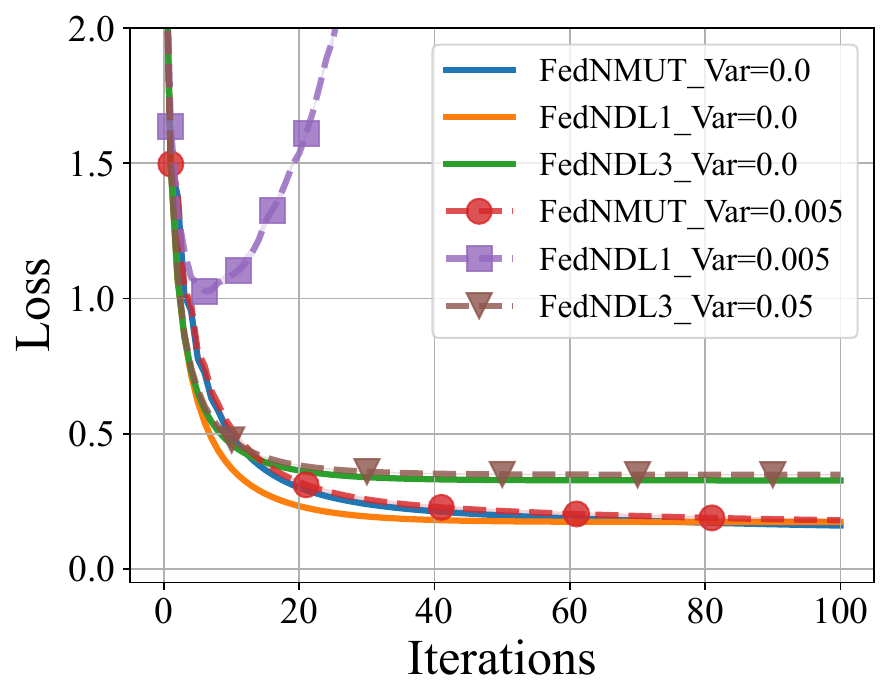}
    \caption{Ring topology} 
    \label{fig:syn_1_0.001}
\end{subfigure} 
\caption{Loss versus iterations with and without noise (Var=0.005) for different communication topologies. $\mu$ = 0.02}
\label{fig:loss_all_1}
\vspace{-4mm}
\end{figure*}

\begin{figure*}[t]
\begin{subfigure}{0.345\textwidth}
    \centering
    \includegraphics[width=\textwidth]{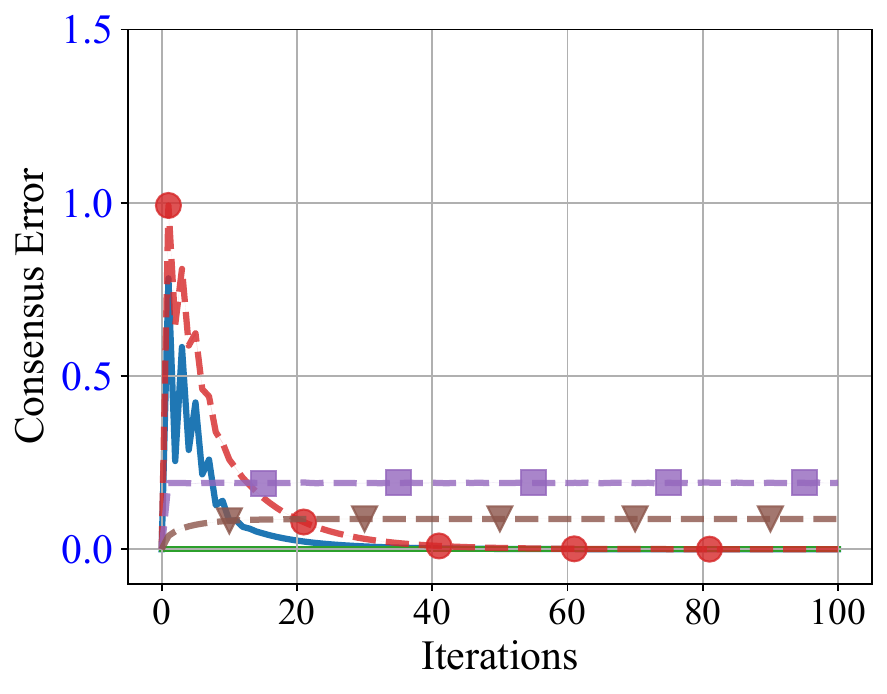}
    \caption{Fully-connected topology} 
    \label{fig:syn_1_0.005_full_cons}
\end{subfigure}
\begin{subfigure}{0.33\textwidth}
    \centering
    \includegraphics[width=\textwidth]{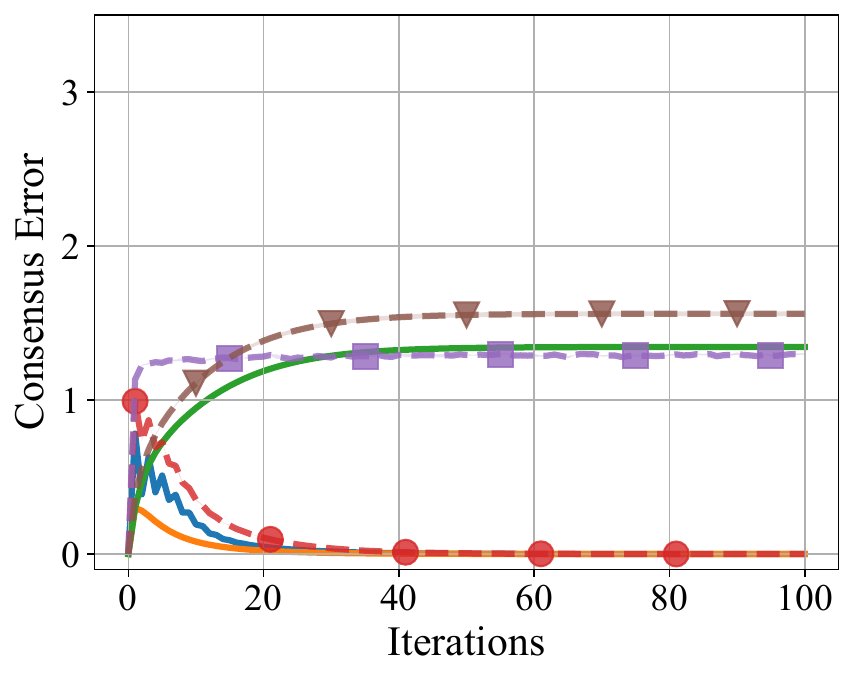}
    \caption{Torus topology}
    \label{fig:syn_1_0.005_torus_cons}
\end{subfigure}
\begin{subfigure}{0.33\textwidth}
    \centering
    \includegraphics[width=\textwidth]{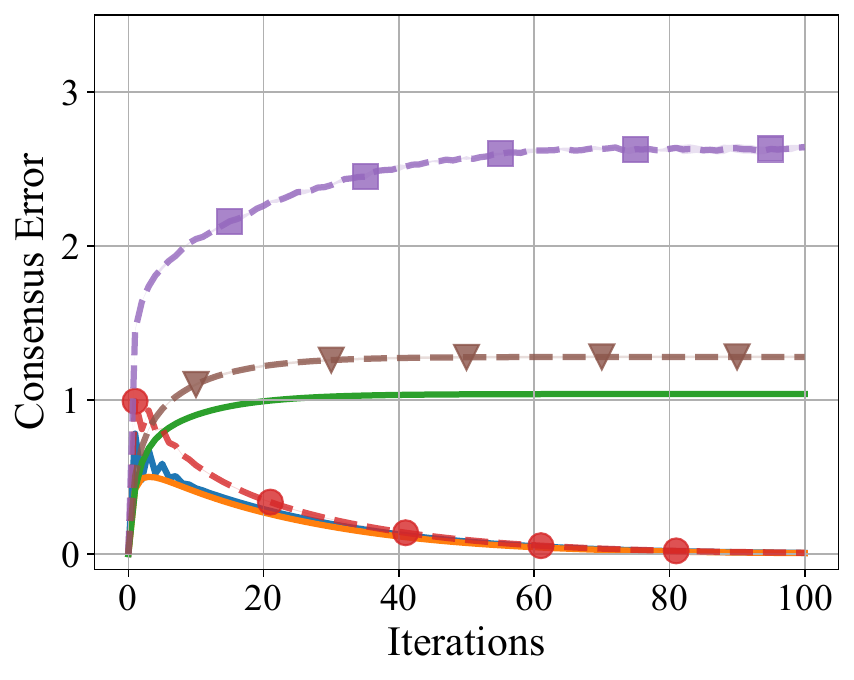}
    \caption{Ring topology} 
    \label{fig:syn_1_0.005_ring_cons}
\end{subfigure}
\caption{Consensus error versus iterations with and without noise (Var=0.005) for different communication topologies. $\mu$ = 0.02. \textit{Note the different Y-axis scale in Fig (a) as compared with Fig (b) and (c) for better readability.}}
\label{fig:consensus_all_1}
\vspace{-4mm}
\end{figure*}

\captionsetup{font=normal}
\captionsetup[sub]{font=normal}
\begin{figure*}[t]
\vspace{0.1in}
\begin{subfigure}{0.33\textwidth}
    \centering    \includegraphics[width=\textwidth]{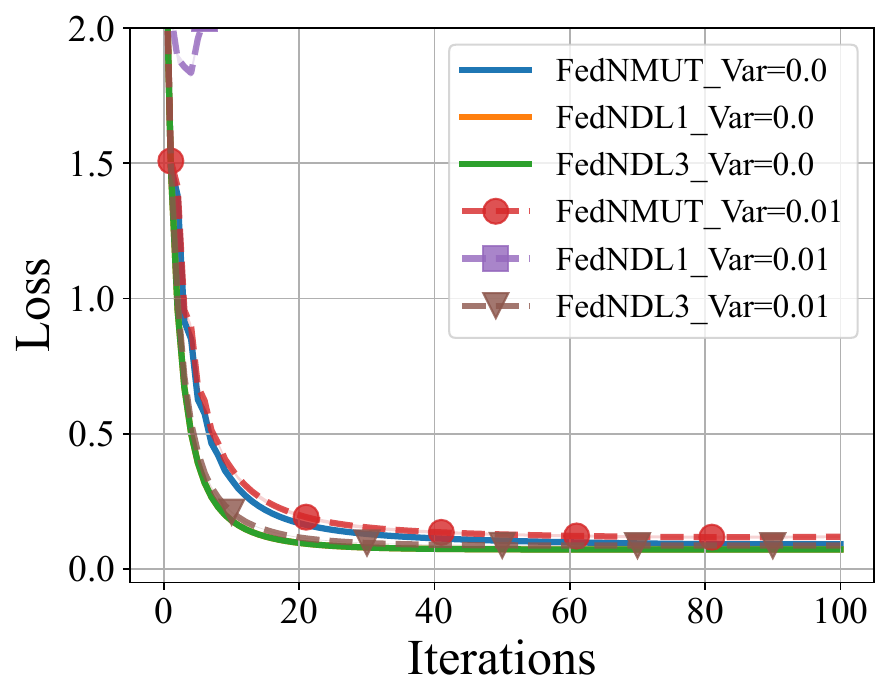}
    \caption{Fully-connected topology}
    \label{fig:syn_1_0}
\end{subfigure}
\begin{subfigure}{0.33\textwidth}
    \centering
    \includegraphics[width=\textwidth]{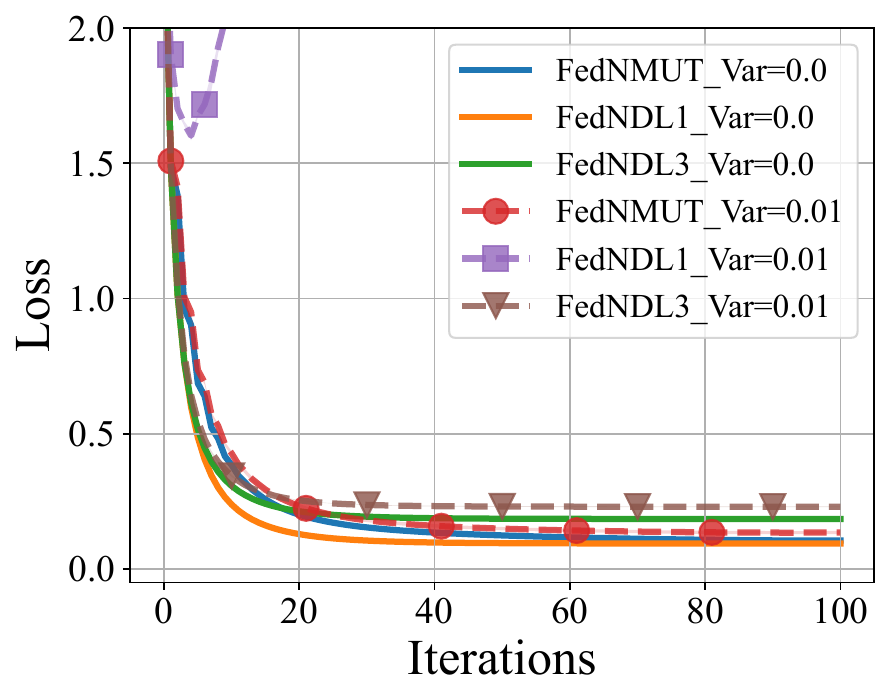}
    \caption{Torus topology}
    \label{fig:syn_1_0.005}
\end{subfigure}
\begin{subfigure}{0.33\textwidth}
    \centering
    \includegraphics[width=\textwidth]{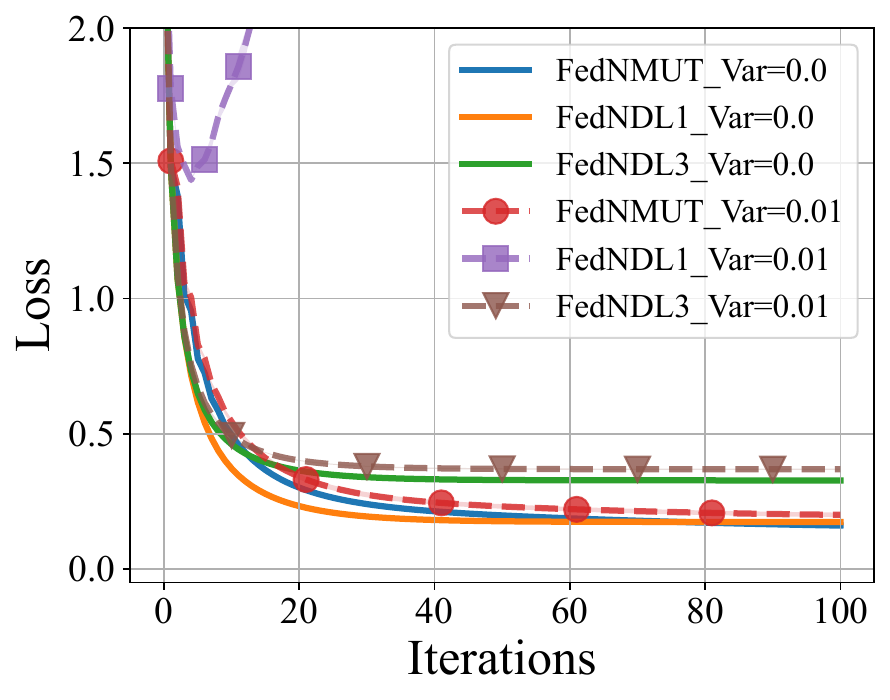}
    \caption{Ring topology} 
    \label{fig:syn_1_0.001}
\end{subfigure} 
\caption{Loss versus iterations with and without noise (Var=0.01) for different communication topologies. $\mu$ = 0.02}
\label{fig:loss_all_2}
\vspace{-4mm}
\end{figure*}

\begin{figure*}[t]
\begin{subfigure}{0.345\textwidth}
    \centering
    \includegraphics[width=\textwidth]{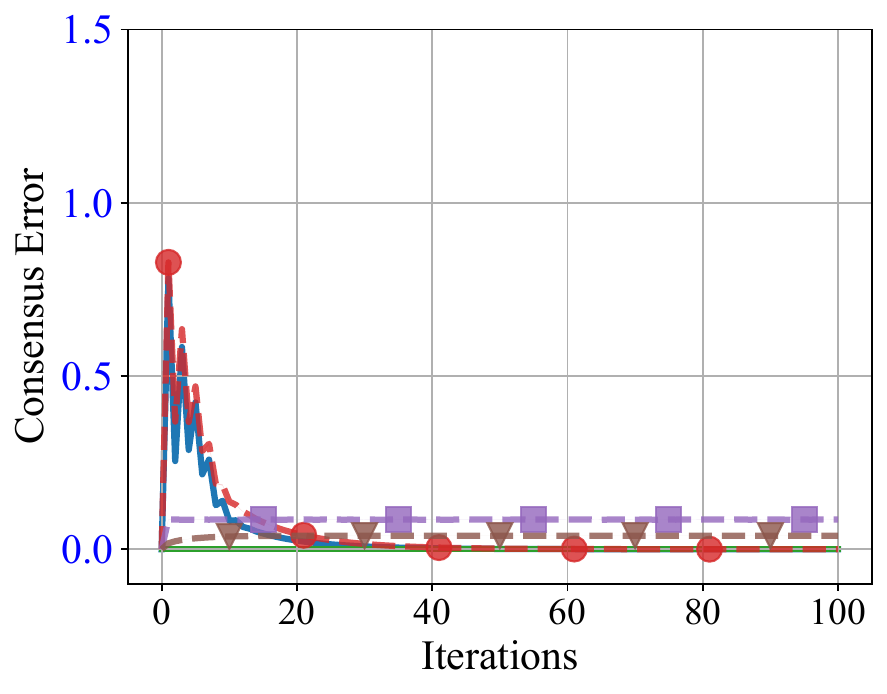}
    \caption{Fully-connected topology} 
    \label{fig:syn_1_0.005_full_cons}
\end{subfigure}
\begin{subfigure}{0.33\textwidth}
    \centering
    \includegraphics[width=\textwidth]{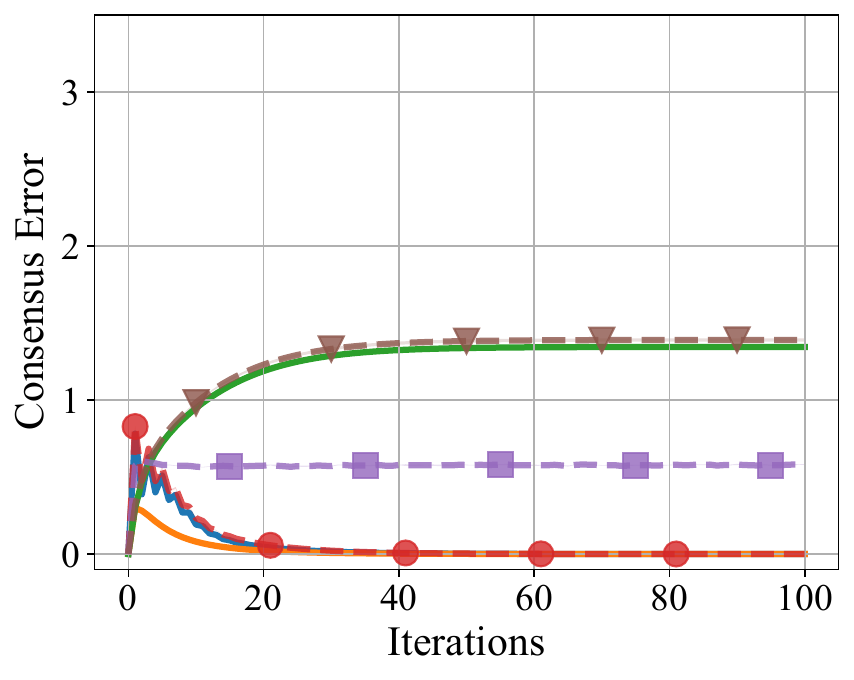}
    \caption{Torus topology}
    \label{fig:syn_1_0.005_torus_cons}
\end{subfigure}
\begin{subfigure}{0.33\textwidth}
    \centering
    \includegraphics[width=\textwidth]{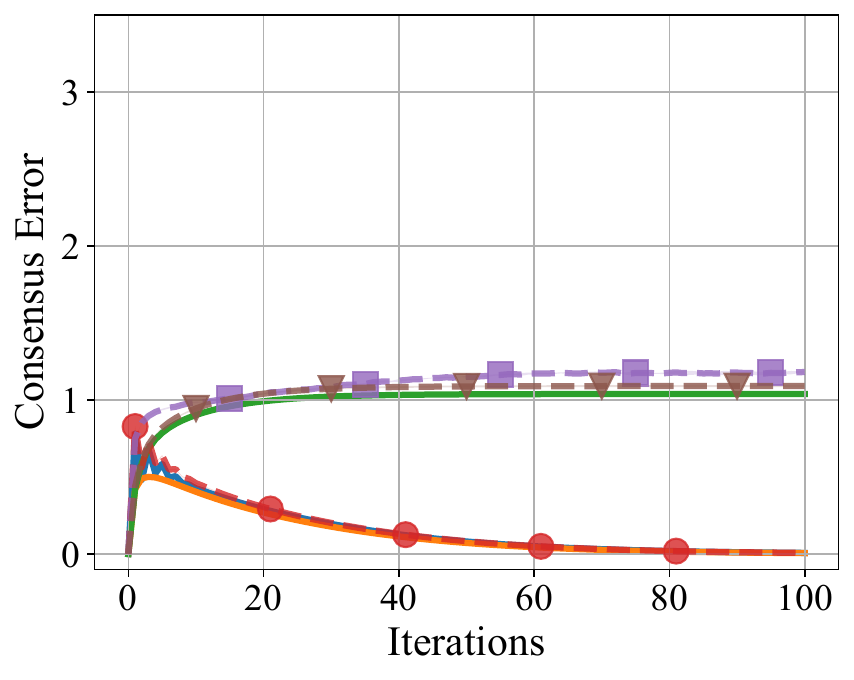}
    \caption{Ring topology} 
    \label{fig:syn_1_0.005_ring_cons}
\end{subfigure}
\caption{Consensus error versus iterations with and without noise (Var=0.01) for different communication topologies. $\mu$ = 0.02. \textit{Note the different Y-axis scale in Fig (a) as compared with Fig (b) and (c) for better readability.}}
\label{fig:consensus_all_2}
\vspace{-4mm}
\end{figure*}

\section{Convergence Analysis}
In this section, we state the main result of this paper that provides an upper bound on the convergence errors of the proposed algorithm. The convergence results are for non-convex $L$-smooth loss functions in the presence of noise. We first present technical results in the form of lemmas that are being used in the proof of the main result of the paper. Detailed proofs of the lemmas and theorem are provided in the appendix of the paper.
\begin{lemma}
\label{lemma1}
Suppose Assumption~\ref{as3} holds and let $\Bar{b}^t = B^t \frac{1}{n}  \mathds{1}$, where $\mathds{1}$ is a vector of all ones, then for all $t$, we have \\
 $$\E[\Bar{b}^t] = 0. $$
\end{lemma}

\begin{lemma}
\label{lemma2}
If Assumptions~\ref{as1}--\ref{as4} are satisfied and $\eta \leq \frac{1}{4L}$, then  \\
\begin{align*}
\mathbb{E} f(\Bar{x}^{t+1}) \leq \mathbb{E} f(\Bar{x}^{t}) + \frac{L \eta^2 \sigma^2}{2n}  -\frac{3\eta}{8} \ \mathbb{E} \|\frac{1}{n} \sum_{i=1}^n  \nabla f_i(x_i^t)\|^2 \\
- \frac{\eta}{2} \ \mathbb{E} \| \nabla f(\Bar{x}^t)\|^2 + \frac{L^2 \eta}{2n} \ \mathbb{E} \|X^t-\Bar{X}^t\|_F^2 + \frac{L \eta^2}{2} \frac{1}{n}\sum_{i=1}^{n}D^2_{t,i}\\
\hspace{2mm}+ \frac{L \mu^2 \eta^2}{2 n} \mathbb{E} \|B^t\|_F^2 .
\end{align*}
\end{lemma}
\begin{lemma}
\label{lemma3}
If Assumptions~\ref{as1}--\ref{as5} are satisfied and $\eta \leq \frac{\rho}{7L}$, then \\  
\begin{align*}
\frac{1}{n} \mathbb{E} ||X^{t+1} - \Bar{X}^{t+1}||_F^2 \leq \frac{1-\rho/4}{n} \mathbb{E} ||X^{t} - \Bar{X}^{t}||_F^2 +\frac{4 \eta^2 \zeta^2}{\rho} \\
 + \frac{4 \eta^2 \sigma^2}{n}  + \frac{6 \eta^2 \mu^2}{\rho n} \mathbb{E} ||B^t||_F^2 + \frac{4\eta^2}{n\rho} D^{2,t}. 
\end{align*}
\end{lemma}
\begin{lemma}   
\label{lemma4}
If Assumptions~\ref{as1}--\ref{as5} are satisfied and $\frac{\mu}{1-\mu} \leq \frac{\rho}{42}$, then  \\
\begin{align*}
 \frac{6 \eta^2 \mu^2}{\rho n (1-\mu)} \mathbb{E} \|B^{t+1}\|_F^2 \leq  \Big(\frac{6 \eta^2 \mu^2}{n \rho (1-\mu)} - \frac{6 \eta^2 \mu^2}{n \rho} \Big) \mathbb{E}\| B^t\|_F^2  \\
 +\frac{\rho}{8n}\mathbb{E}\|X^{t}-\Bar{X}^t\|_F^2  +  \frac{\eta^2 \rho \sigma^2(1-\mu)}{8}  +  \frac{\eta^2 \rho   }{8 n} \sum_{i=1}^{n}{D}^{2}_{t,i} +\frac{ \eta^2 \rho \zeta^2}{ 8}.
\end{align*}
\end{lemma} 
Using the above lemmas, we can state and prove the following theorem.
\begin{theorem}[\textbf{Smooth non-convex cases for FedNMUT}]
\label{thm-all}
If Assumptions~\ref{as1}--\ref{as5} are satisfied and $\eta \leq \min\Big\{\frac{1}{4L}, \frac{\rho}{7L}\Big\}$, $\frac{\mu}{1-\mu} \leq \frac{\rho}{42}$, and $\frac{6 \mu^2 }{\rho (1-\mu)}  \leq \frac{\rho}{8}$, then
\begin{align*}
          \nonumber
        &\frac{1}{T} \sum_{t=0}^{T-1} \mathbb{E} \| \nabla f(\Bar{x}^t)\|^2 \leq \frac{2}{\eta T}(f(\Bar{x}^0  - f^* ) +  \frac{L \mu^2 \eta}{ n T}  \sum_{t=0}^{T-1} \mathbb{E} \|B^t\|_F^2\\
        &\hspace{0.25cm} +  2 L^2 \eta^2 \sigma^2 \left[ \frac{16}{\rho n} + \frac{1-\mu}{2} +\frac{1}{2 n L \eta}\right] + 2 L^2 \eta^2 \zeta^2 \left[\frac{48 }{\rho^2} + \frac{1}{2} \right]\\
        &\hspace{0.25cm}  +  \frac{2 L^2 \eta^2}{T} \sum_{t=0}^{T-1} \sum_{i=1}^{n}{D}^{2}_{t,i} \left[ \frac{16}{n\rho^2} + \frac{1}{2} +\frac{1}{2 n L \eta}  \right] 
\end{align*}
where all the expectations are w.r.t. the data and the noise.
\end{theorem}
\begin{proof}
Combining Lemmas~\ref{lemma3} and \ref{lemma4} and simplifying, we obtain
\begin{multline}
\nonumber
    \frac{1}{n} \mathbb{E} \|X^{t+1} - \Bar{X}^{t+1}\|_F^2 + \frac{6 \eta^2 \mu^2}{n \rho (1-\mu)}\mathbb{E}\|B^{t+1}\|_F^2 \leq \\
    \frac{1-\rho/4}{n} \mathbb{E} \|X^t -\Bar{X}^t\|_F^2 +{4 \eta^2 \sigma^2}+ \frac{12 \eta^2 \zeta^2}{\rho} +\frac{6\eta^2\mu^2}{n\rho} \mathbb{E} \|B^t\|_F^2 \\+ \frac{4\eta^2}{n\rho} \sum_{i=1}^{n}{D}^{2}_{t,i} + \Big(\frac{6 \eta^2 \mu^2}{n \rho (1-\mu)} - \frac{6 \eta^2 \mu^2}{n \rho} \Big) \mathbb{E}\| B^t\|_F^2 \\  +\frac{\rho}{8n}\mathbb{E}\|X^{t}-\Bar{X}^t\|_F^2 +  \frac{\eta^2 \rho \sigma^2(1-\mu)}{8} +  \frac{\eta^2 \rho }{8 n}\sum_{i=1}^{n}{D}^{2}_{t,i} +\frac{ \eta^2 \rho \zeta^2}{8}.
\end{multline}
Simplifying and multiplying the above equation by $\frac{4 L^2 \eta}{\rho}$ gives 
\begin{multline*}
    \frac{4 L^2 \eta}{n \rho} \mathbb{E} \|X^{t+1} - \Bar{X}^{t+1}\|_F^2 + \frac{24 L^2 \eta^3 \mu^2}{n \rho^2 (1-\mu)}\mathbb{E}\|B^{t+1}\|_F^2 
    \leq \\
    \left[ \frac{4 L^2 \eta}{\rho n} - \frac{L^2 \eta}{ n} \right] \mathbb{E} \|X^t -\Bar{X}^t\|_F^2 +\frac{16 L^2 \eta^3 \sigma^2}{\rho n}+ \frac{48 L^2 \eta^3 \zeta^2}{\rho^2} \\
    +\frac{24 L^2 \eta^3\mu^2}{n\rho^2} \mathbb{E} \|B^t\|_F^2 + \frac{16 L^2 \eta^3}{n\rho^2} D^{2,t} +\frac{L^2 \eta}{2n}\mathbb{E}\|X^{t}-\Bar{X}^t\|_F^2 \\ 
    +  \frac{L^2 \eta^3 \sigma^2(1-\mu)}{2} +  \frac{L^2 \eta^3 D^{t,2}}{2} +\frac{ L^2 \eta^3 \zeta^2}{2}\\
        + \Big(\frac{24 L^2 \eta^3 \mu^2}{n \rho^2 (1-\mu)} - \frac{24 L^2 \eta^3 \mu^2}{n \rho^2} \Big) \mathbb{E}\| B^t\|_F^2 
    \end{multline*}
    \begin{multline*}
        \\
        \leq \left[ \frac{4 L^2 \eta}{n \rho} - \frac{L^2 \eta}{2 n} \right] \mathbb{E} \|X^t -\Bar{X}^t\|_F^2  + \frac{24 L^2 \eta^3 \mu^2}{n \rho^2 (1-\mu)} \mathbb{E}\| B^t\|_F^2  \\
        + L^2 \eta^3 \sigma^2 \left[ \frac{16}{\rho n} + \frac{1-\mu}{2}\right] +  L^2 \eta^3 \zeta^2 \left[\frac{48 }{\rho^2} + \frac{1}{2} \right] \\
        +  L^2 \eta^3 \sum_{i=1}^{n}{D}^{2}_{t,i} \left[ \frac{16}{n\rho^2} + \frac{1}{2}  \right].
        \end{multline*}
\\
\begin{eqnarray}
    \label{eq:phi}
    \nonumber
        \text{Let } \ \Phi^t = \frac{4 L^2 \eta}{n \rho} \mathbb{E} \|X^{t} - \Bar{X}^{t}\|_F^2 +\frac{24 L^2 \eta^3 \mu^2}{n \rho^2 (1-\mu)} \mathbb{E} \|B^{t}\|_F^2 \\
        + \mathbb{E}[f(\Bar{x}^t)-f^*].
\end{eqnarray} \\
Combining Lemma~\ref{lemma2} and  Equation \ref{eq:phi} gives
\begin{multline*}
        \! \! \! \! \! \Phi^{t+1} \leq \left[ \frac{4 L^2 \eta}{n \rho} - \frac{L^2 \eta}{2 n} \right] \mathbb{E} \|X^t -\Bar{X}^t\|_F^2  + \frac{24 L^2 \eta^3 \mu^2}{n \rho^2 (1-\mu)} \mathbb{E}\| B^t\|_F^2 \\
         + L^2 \eta^3 \sigma^2 \left[ \frac{16}{\rho n} + \frac{1-\mu}{2}\right] +  L^2 \eta^3 \zeta^2 \left[\frac{48 }{\rho^2}  + \frac{1}{2} \right]  + \frac{L \mu^2 \eta^2}{2 n} \mathbb{E} \|B^t\|_F^2 \\
         +  L^2 \eta^3 D^{t,2} \left[ \frac{16}{n\rho^2} + \frac{1}{2}  \right]    + \frac{L \eta^2 \sigma^2}{2n}  -\frac{3\eta}{8} \ \mathbb{E} \|\frac{1}{n} \sum_{i=1}^n  \nabla f_i(x_i^t)\|^2 \\
         - \frac{\eta}{2} \ \mathbb{E} \| \nabla f(\Bar{x}^t)\|^2    + \frac{L^2 \eta}{2n} \ \mathbb{E} \|X^t-\Bar{X}^t\|_F^2 + \frac{L \eta^2}{2n}D^{2,t} 
    \end{multline*}
    \begin{multline*}
    \le \Phi^{t} +  L^2 \eta^3 \sigma^2 \left[ \frac{16}{\rho n} + \frac{1-\mu}{2} +\frac{1}{2 n L \eta}\right] +  L^2 \eta^3 \zeta^2 \left[\frac{48 }{\rho^2} + \frac{1}{2} \right]  \\
         +  L^2 \eta^3 D^{t,2} \left[ \frac{16}{n\rho^2} + \frac{1}{2} +\frac{1}{2 n L \eta}  \right] + \frac{L \mu^2 \eta^2}{2 n} \mathbb{E} \|B^t\|_F^2 \\
         -\frac{3\eta}{8} \ \mathbb{E} \|\frac{1}{n} \sum_{i=1}^n  \nabla f_i(x_i^t)\|^2 - \frac{\eta}{2} \ \mathbb{E} \| \nabla f(\Bar{x}^t)\|^2. 
\end{multline*} 
Rearranging, we obtain
\begin{multline*} 
        \frac{\eta}{2} \mathbb{E} \| \nabla f(\Bar{x}^t)\|^2 \leq (\Phi^{t}  - \Phi^{t+1} ) +  \frac{L \mu^2 \eta^2}{2 n} \mathbb{E} \|B^t\|_F^2 \\
        +  L^2 \eta^3 \sigma^2 \left[ \frac{16}{\rho n} + \frac{1-\mu}{2} +\frac{1}{2 n L \eta}\right] + L^2 \eta^3 \zeta^2 \left[\frac{48 }{\rho^2} + \frac{1}{2} \right] \\
        +  L^2 \eta^3 D^{t,2} \left[ \frac{16}{n\rho^2} + \frac{1}{2} +\frac{1}{2 n L \eta}  \right]  -\frac{3\eta}{8} \ \mathbb{E} \|\frac{1}{n} \sum_{i=1}^n  \nabla f_i(x_i^t)\|^2.
\end{multline*} 
Telescoping over iterations $t$ from $0$ to $T$ yields
\begin{multline*}
\label{eq:final}
          \nonumber
          \!\!\!\! \!\!\!\!\frac{1}{T} \sum_{t=0}^{T-1} \mathbb{E} \| \nabla f(\Bar{x}^t)\|^2 \leq \frac{2}{\eta T}(f(\Bar{x}^0  - f^* ) +  \frac{L \mu^2 \eta}{ n T}  \sum_{t=0}^{T-1} \mathbb{E} \|B^t\|_F^2 \\
          +  2 L^2 \eta^2 \sigma^2 \left[ \frac{16}{\rho n} + \frac{1-\mu}{2} +\frac{1}{2 n L \eta}\right]
         + 2 L^2 \eta^2 \zeta^2 \left[\frac{48 }{\rho^2}  + \frac{1}{2} \right] \\
         +  \frac{2 L^2 \eta^2}{T} \sum_{t=0}^{T-1} \sum_{i=1}^{n}{D}^{2}_{t,i} \left[ \frac{16}{n\rho^2} + \frac{1}{2} +\frac{1}{2 n L \eta}  \right] .
\end{multline*}

This concludes the proof of the theorem.
\end{proof}
Our theorem establishes a worst-case upper bound on the convergence of the proposed algorithm. The theorem bounds the expected gradient norm, which is a notion of approximate first-order stationarity of the average iterate $\bar{x}_t$. The theorem gives the bound on the convergence which consists of five components: the first component is due to inaccurate initialization. The second term captures the variance of the bias term controlled through the scaling factor $\mu$. The third and fourth terms capture the effect of stochasticity and data heterogeneity. The last term captures the impact of communication noise. 

We have the following corollary for the convergence rate.

\begin{corollary}
\label{corol}
Suppose that the step size $\eta=\mathcal{O}\Big(\sqrt{\frac{n}{T}}\Big)$, then for a sufficiently large $T$, we have 
\begin{align}
\nonumber
        \frac{1}{T} \sum_{t=0}^{T-1} \mathbb{E}||\nabla f(\Bar{x}^t)||^2 =\mathcal{O}\Big(\frac{1}{\sqrt{T}}\Big).
\end{align}
\end{corollary}
\begin{proof}
If the step size $\eta$ is $\mathcal{O}(\sqrt{\frac{n}{T}})$, then we have the following order of convergence for each term in Theorem~\ref{thm-all}:
\begin{gather*}
    \frac{L \mu^2 \eta}{ n T}  \sum_{t=0}^{T-1} \mathbb{E} \|B^t\|_F^2 = \mathcal{O}\left[\frac{1}{ \sqrt{n T}}\right]\Bar{B}^2= \mathcal{O}\left[\frac{1}{ \sqrt{ T}}\Bar{B}^2\right],
    \end{gather*}
    \begin{gather*}
    2 L^2 \eta^2 \sigma^2 \left[ \frac{16}{\rho n} + \frac{1-\mu}{2} +\frac{1}{2 n L \eta}\right] =   \frac{2 L^2 \eta^2 \sigma^2}{\rho n} \\ 
    +  L^2 \eta^2 \sigma^2(1-\mu) +\frac{L \eta \sigma^2}{n}    = \mathcal{O}\left[\frac{1}{\sqrt{ T}} \sigma^2\right],
    \end{gather*}
    \begin{gather*}
    2 L^2 \eta^2 \zeta^2 \left[\frac{48 }{\rho^2} + \frac{1}{2} \right] =   \frac{96 L^2 \eta^2 \zeta^2 }{\rho^2} + L^2 \eta^2 \zeta^2 = \mathcal{O}\left[\frac{1}{T}  \zeta^2\right],
    \end{gather*}
    \begin{gather*}
    \frac{2 L^2 \eta^2}{T} \sum_{t=0}^{T-1} \sum_{i=1}^{n}{D}^{2}_{t,i} \left[ \frac{16}{n\rho^2} + \frac{1}{2} +\frac{1}{2 n L \eta}    \right]= \\\frac{32 L^2 \eta^2 \Bar{D}^2 n T}{n\rho^2 T} +  \frac{L^2 \eta^2 \Bar{D}^2 n T}{T} +\frac{L \eta \Bar{D}^2 n T}{ n T }  = \mathcal{O}\left[\frac{1}{\sqrt{T}} \Bar{D}^2 \right].
\end{gather*}
\\
The overall convergence rate is
\begin{align*}
     \frac{1}{T} \sum_{t=0}^{T-1} \mathbb{E} \| \nabla f(\Bar{x}^t)\|^2 = &\mathcal{O}\Big[\frac{1}{ \sqrt{ T}}\Bar{B}^2+\frac{1}{\sqrt{ T}} \sigma^2 \\
     &\hspace{5mm} +\frac{1}{T}  \zeta^2+\frac{1}{\sqrt{T}} \Bar{D}^2 \Big].
\end{align*}

Therefore, at large $T$, the convergence rate of \textit{FedNMUT} is\textbf{ $\mathcal{O}\Big(\frac{1}{\sqrt{T}}\Big)$}.
\end{proof}
It follows from the Corollary~\ref{corol} that the FedNMUT algorithm can achieve a linear speedup, exhibiting a convergence rate of $\mathcal{O}\Big(\frac{1}{\sqrt{T}}\Big)$, given that $T$ is large enough and unaffected by the communication topology. 
This rate of convergence is comparable to the best-known results for decentralized SGD algorithms found in existing literature\cite{xin2020decentralized,nedic2014distributed}.

\section{Numerical Experiments}\label{sec:exp}
We conducted a series of experiments on regression tasks to evaluate the effects of noise on the convergence behavior of the proposed algorithm. We set the number of clients $n=16$. Each experiment is executed thrice, with the outcomes (loss/consensus error) being averaged. Our approach employs a mean-squared error loss function enhanced with $L_2$ regularization. We set the initial learning rate at 0.2, applying a decay factor of 0.9 in subsequent iterations. We create data samples ($m$ = 10,000) in the form of ${(x_i; y_i)}^{m}_{i=1}$, modeled as $y_i = \langle w, x_i \rangle + \epsilon_i$, where $w$ belongs to $\mathbb{R}^{2000}$, $x_i$ follows a normal distribution $\mathcal{N}(0; I_{2000})$, and noise $\epsilon_i$ adheres to $\mathcal{N}(0, 0.05)$.

Our experimentation spans various noise variances, $D_{t,i}^2$, across all $t,i$, reflecting the conditions specified in the algorithm. These tests encompass different communication topologies, specifically ring, torus, and fully connected networks. In these topologies, the nonzero elements of the mixing matrix hold the values of $\frac{1}{3}$, $\frac{1}{5}$, and $\frac{1}{n}$, respectively.

We commence with noise-free trials to establish a baseline, subsequently incrementing noise variance to examine algorithmic robustness. For consistency, the experimental findings with noise variances of $D^{2}_{t} = 0.005$ and $D^{2}_{t} = 0.01$ are depicted alongside the no-noise condition in \Cref{fig:loss_all_1,fig:consensus_all_1,fig:loss_all_2,fig:consensus_all_2}. These tests were conducted using Intel's Xeon Gold workstation.

We observe in~\Cref{fig:loss_all_1,fig:loss_all_2} that the FedNDL1 algorithm performs poorly in terms of convergence due to the presence of noise compared to FedNDL3 and FedNMUT. The FedNDL1 and FedNDL2 perform similarly. The FedNMUT algorithm outperforms FedNDL3 in presence of noise.


The consensus error depends on the topology of the communication network. We observed that the consensus error is low for the fully connected network and high for the ring topology for the same algorithm in the presence of noise which is also consistent with the~\Cref{thm-all}. The consensus error function plots in~\Cref{fig:consensus_all_1,fig:consensus_all_2} are consistent with the connectivity of the communication network (number of client interactions). The fully connected topology encompasses the maximum number of clients interaction. It therefore, yields the lowest consensus error, followed by the torus topology and then the ring topology, which has the lowest number of client interactions.

\section{Summary and Conclusions}
We studied the impact of noisy communication channels on the convergence of Decentralized Federated Learning with Model Update Tracking approach. We proposed multiple scenarios for establishing consensus in the presence of noise and provided experimental results from our algorithm testing. Additionally, we provided theoretical results for FedNMUT under the assumption of smooth non-convex function, and we observed that the noise term in the upper bound given by~\Cref{thm-all} is of order $\mathcal{O}(\frac{1}{\sqrt{n T}})$ and is independent of communication topology. We conducted numerical experiments and observed that FedNMUT is more robust against the added noise than the existing state-of-the-art algorithms that include communication noise.

\bibliographystyle{ieeetr}
\bibliography{paper_journal}

\clearpage
\appendix
\setcounter{lemma}{0}
\setcounter{assumption}{0}
\section{Assumptions, Definitions, and Background Results}
\label{appendix:convergence_rate_proofs}
In this work, we solve the optimization problem of minimizing global loss function $f(x)$ distributed across $n$ clients as given below. 
Note that $F_i$ is a local loss function defined in terms of the data sampled ($\xi_i$) from the local dataset $D_i$ at client $i$.
\begin{equation*}
\begin{split}
    \min \limits_{x \in \mathbb{R}^d} f(x) &= \frac{1}{n}\sum_{i=1}^n f_i(x_i),
\end{split}
\end{equation*}
\begin{equation*}
    \label{eq:local-equation}
    f_i(x_i) = \E_{\xi_i \sim \mathcal{D}_{i}}[F_i(x_i,\xi_i)].
\end{equation*}
We reiterate the update scheme of \textit{FedNMUT} in a matrix form:
\begin{align} 
\label{eq:appendix_our_scheme_matrix_form}
	\begin{split}
		X^{t+1} 	&= X^t - \eta \widetilde{Y}^t = X^t - \eta ({Y}^t+\delta^t)\\
    {Y}^t & = \Delta^t + \mu [W \widetilde{Y}^{t-1} - \frac{1}{\eta} (W-I) X^t - \Delta^{t-1}] \\
        \Delta^t & = G^t - \frac{1}{\eta} (W-I) X^t,
	\end{split}
\end{align}
where $W$ is the mixing matrix, $I$ is the identity matrix, $X=[x_1,x_2, \hdots,x_n] \in \mathbb{R}^{d \times n}$ is the matrix containing model parameters, $x_i \in \mathbb{R}^{d}$ is model parameters of client $i$, $Y=[y_1,y_2, \hdots,y_n] \in \mathbb{R}^{d \times n}$ is the matrix containing tracking variables, $G=[g_1,g_2, \hdots,g_n] \in \mathbb{R}^{d \times n}$ is the matrix containing local gradients, $\widetilde{Y}=[\widetilde{y_1},\widetilde{y_2}, \hdots,\widetilde{y_n}] \in \mathbb{R}^{d \times n}$ is the matrix containing tracking variables with noises from communication channels, $\delta = [\delta_1,\delta_2, \hdots,\delta_n]$, $G=[g_1,g_2, \hdots,g_n] \in \mathbb{R}^{d \times n}$ is the matrix containing local gradients, $\mu$ is the \textit{FedNMUT} scaling factor, $\eta$ is the learning rate. Now, we rewrite the above equation in the form of a bias correction update and communication noise,

\begin{align} 
\label{eq:appendix_bias_matrix_form}
	\begin{split}
		X^{t+1} 	&= W X^t - \eta (G^t+\mu B^t + \delta^{t}) \\
		B^t & = -\frac{1}{\eta} [(2W-I) (X^{t}-X^{t-1}) +\eta G^{t-1} ]. \\
	\end{split}
 \end{align}

\subsection{Assumptions}
We now discuss the assumptions made in our analysis of the algorithms. 
\begin{assumption}[\textbf{Smoothness}]
\label{as1} 
The objective function $F_i(x, \xi)$ is $L$-smooth with respect to $x$, for all $\xi$. Each $f_i(x)$ is $L$-smooth, that is,
\begin{equation}
\|\nabla f_i (x) - \nabla f_i (y)\| \leq L \|x-y\|, \qquad \mbox{for all x,y} .
\nonumber
\end{equation}
Hence the function $f$ is also $L$-smooth.
\end{assumption}
\begin{assumption}[\textbf{Bounded Variance}] 
\label{as2} 
The variance of the stochastic gradient of each client $i$ is bounded, that is, $$\E[||{\nabla}F_i (x_i^t,\xi_i^t) - \nabla f_i (x_i^t)||^2] \leq \sigma^2,$$ where $\xi_i^t$ denotes random batch of samples in client node $i$ for $t^{th}$ round, and ${\nabla}F_i(x_i^t,\xi_i^t)$ denotes the stochastic gradient. In addition,  we also assume that the stochastic gradient is unbiased, i.e., $\E[{\nabla}F_i (x_i^t,\xi_i^t)] = \nabla f_i (x_i^t)$.
\end{assumption}
\begin{assumption}[\textbf{Mixing matrix}] 
\label{as3} 
 For $\rho \in (0,1]$, the mixing matrix W satisfies,
\begin{equation*}
    \|(\Bar{X} - X)W \|_{F}^2 \leq (1 - \rho)\|\Bar{X} - X\|_{F}^2,
\end{equation*}
which means that the gossip averaging step brings the columns of $X \in \mathbb{R}^{d \times n}$ closer to the row-wise average, that is, $\Bar{X}= X\frac{\mathds{1}\mathds{1}^{\top}}{n}$.
\end{assumption}
{Note that standard topologies such as ring,  torus, and	fully-connected satisfy the above assumption.} 
\begin{assumption}[\textbf{Noise model}] 
\label{as4} 
The noise present due to contamination of communication channel $\delta_i^{(t)}$ is independent, has zero mean and bounded variance, that is, \\
\hspace{10mm}$\E[\delta_i^{(t)}] = 0$ and $\E[||\delta_i^{(t)}||^2] = D_{t,i}^2 < \infty$. 
\end{assumption}

\begin{assumption}[\textbf{Bounded Client Dissimilarity (BCD)}] 
\label{as5} 
For all $x \in \mathbb{R}^{d}$, where $\zeta$ is a constant,
\begin{equation*}
    \label{eq:bcd}
    \frac{1}{n}\sum_{i=1}^{n}\|\nabla f_{i}(x) - \nabla f(x)\|^2 \leq \zeta^2.
\end{equation*}
\end{assumption}
The above assumption is made to limit the extent of client heterogeneity and is standard in the DFL setup. 
\\
Further, we define the average gradients $\Bar{g}^t = \frac{1}{n} \sum_{i=1}^n \nabla F_i(x_i^t, \xi_i^t)$ where $\xi_i^t$ is sampled mini-batch of data on client $i$.

\section{Technical Results}
\begin{lemma}
\label{lemma1}
Suppose Assumption~\ref{as3} holds and let $\Bar{b}^t = B^t \frac{1}{n}  \mathbbm{1}$, where $\mathbbm{1}$ is a vector of all ones, then for all $t$, we have $\E[\Bar{b}^t] = 0.$ 
\end{lemma}

\begin{proof}
	Starting from the definition of $B^t$
	\begin{align*}
    &B^t = -\frac{1}{\eta} [(2W-I)(X^t-X^{t-1})+\eta G^{t-1}]\\
    & \text{multiply } \frac{1}{n}  \mathbbm{1} \text{ on both sides}\\
    &\Bar{b}^t =  -\frac{1}{\eta} [\Bar{x}^t-\Bar{x}^{t-1}+\eta \Bar{g}^{t-1}] \hspace{4mm} (\because (2W-I)\mathbbm{1}=\mathbbm{1})\\
    & \text{now, multiplying } \frac{1}{n} \mathbbm{1} \text{ to  } X^{t+1}= W X^t - \eta (G^t+\mu B^t + \delta^t)\\
    &\Bar{x}^{t+1} = \Bar{x}^{t} - \eta \Bar{g}^t - \eta \mu \Bar{b}^t -\eta \Bar\delta^t \implies \Bar{x}^{t} -\Bar{x}^{t-1} + \eta \Bar{g}^{t-1} = - \eta \mu \Bar{b}^{t-1} - \eta \Bar\delta^{t-1}\\
    &\implies \Bar{b}^t = \mu \Bar{b}^{t-1} + \Bar\delta^{t-1}
    \end{align*}
    Now, given that $\Bar{b}^0=0$ and taking the expectation $\Bar{b}^t$ w.r.t noise, we have
    \begin{align*}
        \E[\Bar{b}^t]  = 0.
    \end{align*}
\end{proof}


\begin{lemma} 
\label{lemma2}
Given assumptions 1-3 and $\eta \leq \frac{1}{4L}$, we have  \\
 $\mathbb{E} f(\Bar{x}^{t+1}) \leq \mathbb{E} f(\Bar{x}^{t}) + \frac{L \eta^2 \sigma^2}{2n}  -\frac{3\eta}{8} \ \mathbb{E} \|\frac{1}{n} \sum_{i=1}^n  \nabla f_i(x_i^t)\|^2 - \frac{\eta}{2} \ \mathbb{E} \| \nabla f(\Bar{x}^t)\|^2 + \frac{L^2 \eta}{2n} \ \mathbb{E} \|X^t-\Bar{X}^t\|_F^2 + \frac{L \eta^2}{2} \frac{1}{n}\sum_{i=1}^{n}D^2_{t,i} + \frac{L \mu^2 \eta^2}{2 n} \mathbb{E} \|B^t\|_F^2 $.

\end{lemma}

\begin{proof}
	From the definition of $X^{t+1}$, we have
 \begingroup
\allowdisplaybreaks

	\begin{align*}
    &X^{t+1} = WX^t - \eta [G^t+\mu B^t] - \eta \delta^{t} \\
    \implies & \Bar{x}^{t+1} = \Bar{x}^t - \eta \Bar{g}^t - \eta \mu \Bar{b}^t - \eta \Bar{\delta}^{t} \hspace{4mm} 
    \end{align*}
    
    using L-smoothness assumption 
    \begin{align*}
    \mathbb{E} f(\Bar{x}^{t+1}) & \leq \mathbb{E} f(\Bar{x}^{t}) + \mathbb{E} \langle \nabla f(\Bar{x}^t), \Bar{x}^{t+1}-\Bar{x}^t \rangle + \frac{L}{2} \mathbb{E} \|\Bar{x}^{t+1}-\Bar{x}^t\|^2 \\
    & =  \mathbb{E} f(\Bar{x}^{t}) + \mathbb{E} \langle \nabla f(\Bar{x}^t), -\eta \Bar{g}^t - \eta \Bar{\delta}^{t} - \eta \mu \Bar{b}^t\rangle + \frac{L\eta^2}{2} \mathbb{E} \|\Bar{g}^t + \Bar{\delta}^{t} + \mu \Bar{b}^t \|^2 \\
    \end{align*}

    \begin{align*}
    \mathbb{E} f(\Bar{x}^{t+1}) =  \mathbb{E} f(\Bar{x}^{t}) - \eta \mathbb{E} \langle \nabla f(\Bar{x}^t),  \mathbb{E}[\Bar{g}^t] \rangle -\underbrace{ \eta \mathbb{E} \langle \nabla f(\Bar{x}^t),  \Bar{\delta}^{t} \rangle }_{Term(A)} - \underbrace{ \eta \mu  \mathbb{E} \langle \nabla f(\Bar{x}^t),\Bar{b}^t \rangle }_{Term(B)}+  \frac{L \eta^2}{2} \mathbb{E} \|\frac{1}{n} \sum_{i=1}^n \nabla F_i(x_i^t)+  \Bar{\delta}^{t} + \mu \Bar{b}^t \|^2
    \end{align*}

    Term (A):
\\
Taking the expectation of term (A) w.r.t noise, we get
\begin{flalign*}
    \E[A] = 0
\end{flalign*}

Term (B):
\begin{align*}
B &= \eta \mu  \mathbb{E} \langle \nabla f(\Bar{x}^t),\Bar{b}^t \rangle \\
&= 0 \ (\because \E[\Bar{b}^t]=0  \ \text{from Lemma~\ref{lemma1}})
&
\end{align*}

    Now,     
    \begin{align*}
    \mathbb{E} f(\Bar{x}^{t+1}) =  \mathbb{E} f(\Bar{x}^{t}) \underbrace{- \eta \frac{1}{n} \sum_{i=1}^n \mathbb{E} \langle \nabla f(\Bar{x}^t),  \nabla f_i(x_i^t) \rangle}_{Term(C)} \underbrace{ + \frac{L \eta^2}{2} \mathbb{E} \|\frac{1}{n} \sum_{i=1}^n (\nabla F_i(x_i^t) + \Bar{\delta}^{t}\|^2}_{Term(D)}\underbrace{ + \frac{L \eta^2}{2} \mathbb{E} \|\mu \Bar{b}^t\|^2}_{Term(E)} \\    
    \end{align*}   
    \\

    Term (C):
    \begin{align*}    
    C &= - \eta \frac{1}{n} \sum_{i=1}^n \mathbb{E} \langle \nabla f(\Bar{x}^t),  \nabla f_i(x_i^t) \rangle \\
    &= - \eta \mathbb{E} \langle \nabla f(\Bar{x}^t),   \frac{1}{n} \sum_{i=1}^n \nabla f_i(x_i^t) \rangle \\
    & \overset{(a)}{=} - \frac{\eta}{2} \mathbb{E} \| \nabla f(\Bar{x}^t)\|^2 - \frac{\eta}{2} \mathbb{E} \|\frac{1}{n} \sum_{i=1}^n  \nabla f_i(x_i^t)\|^2 + \frac{\eta}{2} \mathbb{E} \|\frac{1}{n} \sum_{i=1}^n (\nabla f_i(x_i^t)- \nabla f_i(\Bar{x}^t) )\|^2 \\
    &  \overset{(b)}{\leq} -\frac{\eta}{2} \mathbb{E} \|\frac{1}{n} \sum_{i=1}^n  \nabla f_i(x_i^t)\|^2 - \frac{\eta}{2} \mathbb{E} \| \nabla f(\Bar{x}^t)\|^2 + \frac{L^2 \eta}{2n} \sum_{i=1}^n\mathbb{E} \|x_i^t-\Bar{x}^t\|^2\\ 
    &  \leq  -\frac{\eta}{2} \mathbb{E} \|\frac{1}{n} \sum_{i=1}^n  \nabla f_i(x_i^t)\|^2 - \frac{\eta}{2} \mathbb{E} \| \nabla f(\Bar{x}^t)\|^2+ \frac{ L^2 \eta}{2n} \ \mathbb{E} \|X^t-\Bar{X}^t\|_F^2\\
    \end{align*}
     (a) uses the fact that $-2\langle a,b \rangle = -\|a\|^2-\|b\|^2+\|a-b\|^2$. (b) uses L-smoothness. 
    \\
    \\
    Term (D):

    \begin{align*}    
    D &= \frac{L \eta^2}{2}\E\left[\|\frac{1}{n} \sum_{i=1}^n \nabla F_i(x_i^t) + \Bar{\delta}^{t}\|^2 \right] \\
    &= \frac{L \eta^2}{2} \E\left[ \|\frac{1}{n} \sum_{i=1}^n (\nabla F_i(x_i^t) \|^2 + \|\frac{1}{n} \sum_{i=1}^n {\delta_i}^{t} \|^2 + 2\langle\frac{1}{n}\sum_{i=1}^{n}\nabla F_i(x_i^t) , \frac{1}{n}\sum_{i=1}^{n}\delta_{t,i} \rangle \right] \\  
    &= \frac{L \eta^2}{2}\left[ \E[\|\frac{1}{n}\sum_{i=1}^{n} \nabla F_i(x_i^t)\|^2] +  \frac{1}{n}\sum_{i=1}^{n}D^2_{t,i}\right] \\
    \nonumber
    & = \frac{L \eta^2}{2}\E[\|\frac{1}{n}\sum_{i=1}^{n} [\nabla F_i(x_i^t) - \nabla f_i(x_i^t) +\nabla f_i(x_i^t)]\|^2] + \frac{L \eta^2}{2}\frac{1}{n}\sum_{i=1}^{n}D^2_{t,i}
    \\
    \nonumber
    & = \frac{L \eta^2}{2}\E \Big[ \|\frac{1}{n}\sum_{i=1}^{n} [\nabla F_i(x_i^t) - \nabla  f_i(x_i^t) ]\|^2 + \|\frac{1}{n}\sum_{i=1}^{n}\nabla f_i(x_i^t)\|^2 \\
    \nonumber
    &\hspace{1cm} + 2\langle\frac{1}{n}\sum_{i=1}^{n }\nabla F_i(x_i^t) - \nabla  f_i(x_i^t),\frac{1}{n}\sum_{i=1}^{n}  \nabla  f_i(x_i^t) \rangle \Big] + \frac{L \eta^2}{2}\frac{1}{n}\sum_{i=1}^{n}D^2_{t,i}
    \\
    & \overset{(a)}\leq \frac{L \eta^2}{2}\Big(\frac{\sigma^2}{n} + \E[\|\frac{1}{n}\sum_{i=1}^{n} \nabla f_i(x_i^t)\|^2]\Big) + \frac{L \eta^2}{2}\frac{1}{n}\sum_{i=1}^{n}D^2_{t,i}
\end{align*}
(a) results by using  Assumption~\ref{as2}.
\\
\\
\\
    Term (E):
    \begin{align*}
      E &=  \frac{L \eta^2}{2} \mathbb{E} \|\mu \Bar{b}^t\|^2 \\
      &= \frac{L \mu^2 \eta^2}{2 } \mathbb{E}  \| \frac{1}{n}\sum_{i=1}^n  {b}^t_i\|^2 \\
    & \leq  \frac{L \mu^2 \eta^2}{2 } \frac{1}{n}\sum_{i=1}^n \mathbb{E}  \|  {b}^t_i\|^2    \\
    & \leq  \frac{L \mu^2 \eta^2}{2 n} \mathbb{E} \|B^t\|_F^2   
    \end{align*}
\\

Now putting together Term C, Term D and Term E:
    \begin{align*}
    \mathbb{E} f(\Bar{x}^{t+1}) &\le  \mathbb{E} f(\Bar{x}^{t}) \underbrace{- \eta \frac{1}{n} \sum_{i=1}^n \mathbb{E} \langle \nabla f(\Bar{x}^t),  \nabla f_i(x_i^t) \rangle}_{Term(C)} \underbrace{ + \frac{L \eta^2}{2} \mathbb{E} \|\frac{1}{n} \sum_{i=1}^n (\nabla F_i(x_i^t) + \Bar{\delta}^{t}\|^2}_{Term(D)} \underbrace{ + \frac{L \eta^2}{2} \mathbb{E} \|\mu \Bar{b}^t\|^2}_{Term(E)}\\
    &\le  \mathbb{E} f(\Bar{x}^{t}) -\frac{\eta}{2} \mathbb{E} \|\frac{1}{n} \sum_{i=1}^n  \nabla f_i(x_i^t)\|^2 - \frac{\eta}{2} \mathbb{E} \| \nabla f(\Bar{x}^t)\|^2+ \frac{ L^2 \eta}{2n} \ \mathbb{E} \|X^t-\Bar{X}^t\|_F^2 \\ 
    & \hspace{2cm} +  \frac{L \eta^2}{2}\Big(\frac{\sigma^2}{n} + \E[\|\frac{1}{n}\sum_{i=1}^{n} \nabla f_i(x_i^t)\|^2]\Big) + \frac{L \eta^2}{2}\frac{1}{n}\sum_{i=1}^{n}D^2_{t,i} + \frac{L \mu^2 \eta^2}{2 n} \mathbb{E} \|B^t\|_F^2  \\
    & \le \mathbb{E} f(\Bar{x}^{t}) + \frac{L \eta^2 \sigma^2}{2n} + (\frac{L\eta^2}{2}-\frac{\eta}{2}) \ \mathbb{E} \|\frac{1}{n} \sum_{i=1}^n  \nabla f_i(x_i^t)\|^2 - \frac{\eta}{2} \ \mathbb{E} \| \nabla f(\Bar{x}^t)\|^2  \\
    &  \hspace{4mm} + \frac{L^2 \eta}{2n} \ \mathbb{E} \|X^t-\Bar{X}^t\|_F^2 + \frac{L \eta^2}{2}\frac{1}{n}\sum_{i=1}^{n}D^2_{t,i} + \frac{L \mu^2 \eta^2}{2 n} \mathbb{E} \|B^t\|_F^2  \\
    & \overset{(a)}\leq \mathbb{E} f(\Bar{x}^{t}) + \frac{L \eta^2 \sigma^2}{2n}  -\frac{3\eta}{8} \ \mathbb{E} \|\frac{1}{n} \sum_{i=1}^n  \nabla f_i(x_i^t)\|^2 - \frac{\eta}{2} \ \mathbb{E} \| \nabla f(\Bar{x}^t)\|^2 \\
    &  \hspace{4mm} + \frac{L^2 \eta}{2n} \ \mathbb{E} \|X^t-\Bar{X}^t\|_F^2 + \frac{L \eta^2}{2} \frac{1}{n}\sum_{i=1}^{n}D^2_{t,i} + \frac{L \mu^2 \eta^2}{2 n} \mathbb{E} \|B^t\|_F^2 
    \end{align*}   
    \endgroup
(a) follows from the assumption that \textbf{$\eta \leq \frac{1}{4L}$}
\end{proof}

Now, we proceed to obtain a bound on the consensus error. 

\begin{lemma} 
\label{lemma3}
Given assumptions 1-3 and $\eta \leq \frac{\rho}{7L}$, we have  \\
 $\frac{1}{n} \mathbb{E} \|X^{t+1} - \Bar{X}^{t+1}\|_F^2 \le \frac{1-\rho/4}{n} \mathbb{E} \|X^t -\Bar{X}^t\|_F^2 +{4 \eta^2 \sigma^2}+ \frac{12 \eta^2 \zeta^2}{\rho}+\frac{6\eta^2\mu^2}{n\rho} \mathbb{E} \|B^t\|_F^2 + \frac{4\eta^2}{n\rho} \sum_{i=1}^{n}{D}^{2}_{t,i}$
 \end{lemma}

\begin{proof}
	Starting from the update step \ref{eq:appendix_bias_matrix_form}
 \begingroup
\allowdisplaybreaks
	\begin{align*}
    \frac{1}{n} \mathbb{E} \|X^{t+1} - \Bar{X}^{t+1}\|_F^2  = & \frac{1}{n} \mathbb{E} \|WX^t - \eta [G^t+\mu B^t+ \delta^t] -(\Bar{X}^t-\eta \Bar{G}^t - \eta \Bar{\delta}^t)\|_F^2\\
    = & \frac{1}{n} \mathbb{E} \|WX^t -\Bar{X}^t - \eta (G^t-\Bar{G}^t)-\eta \mu B^t - \eta (\delta^t-\Bar{\delta}^t)\|_F^2\\        
    = & \frac{1}{n} \mathbb{E} \|WX^t -\Bar{X}^t - \eta (G^t-\Bar{G}^t+\mathbb{E}[G^t]-\mathbb{E}[G^t]+\mathbb{E}[\Bar{G}^t]-\mathbb{E}[\Bar{G}^t])\\
    &\hspace{8mm} -\eta \mu B^t - \eta (\delta^t -\Bar{\delta}^t)\|_F^2\\
    = & \frac{1}{n} \mathbb{E} \|\underbrace{WX^t -\Bar{X}^t - \eta (\mathbb{E}[G^t]-\mathbb{E}[\Bar{G}^t])-\eta \mu B^t }_{(A)} \underbrace{-\eta (G^t - \mathbb{E}[G^t])}_{(B)} \\ &\hspace{8mm} +\underbrace{\eta(\Bar{G}^t-\mathbb{E}[\Bar{G}^t])}_{(C)} - \underbrace{ \eta (\delta^t -\Bar{\delta}^t)}_{(D)} \|_F^2\\
    \overset{(a)}{=}& \frac{1}{n} \Big[ \mathbb{E} \|WX^t -\Bar{X}^t - \eta (\mathbb{E}[G^t]-\mathbb{E}[G^t] -\eta \mu B^t  \|_F^2  \\
    &\hspace{8mm}+ \mathbb{E} \|\eta (G^t-\mathbb{E}[G^t])\|_F^2 + \mathbb{E} \|\eta (\Bar{G}^t-\mathbb{E}[\Bar{G}^t])\|_F^2 + \mathbb{E} \|\eta ( \delta^t -\Bar{\delta}^t)\|_F^2\Big]\\
    \overset{(b)}{\leq} & \frac{1}{n} \mathbb{E} \|WX^t -\Bar{X}^t - \eta (\mathbb{E}[G^t]-\mathbb{E}[\Bar{G}^t])-\eta \mu B^t\|_F^2+ {4 \eta^2 \sigma^2} +\frac{\eta^2}{n} \mathbb{E} \|\delta^t-\Bar{\delta}^t\|_F^2\\
    \overset{(c)}{\leq} & \frac{1+\rho/2}{n} \mathbb{E} \|WX^t -\Bar{X}^t\|_F^2 + \frac{\eta^2 (1+2/\rho)}{n} \mathbb{E} \|\mathbb{E}[G^t]-\mathbb{E}[\Bar{G}^t]-\mu B^t\|_F^2\\
    &+{4 \eta^2 \sigma^2}+\frac{\eta^2}{n} \mathbb{E} \|\delta^t-\Bar{\delta}^t\|_F^2\\
    \overset{(d)}{\leq} & \frac{(1-\rho)(1+\rho/2)}{n} \mathbb{E} \|X^t -\Bar{X}^t\|_F^2 + \frac{3\eta^2}{n\rho} \mathbb{E} \|\mathbb{E}[G^t]-\mathbb{E}[\Bar{G}^t]-\mu B^t )\|_F^2\\
    &+{4 \eta^2 \sigma^2}+\frac{\eta^2}{n} \mathbb{E} \|\delta^t-\Bar{\delta}^t\|_F^2\\
    \leq & \frac{(1-\rho)(1+\rho/2)}{n} \mathbb{E} \|X^t -\Bar{X}^t\|_F^2 +{4 \eta^2 \sigma^2} + \frac{6\eta^2}{n\rho} \mathbb{E} \|\mathbb{E}[G^t]-\mathbb{E}[\Bar{G}^t]\|_F^2\\
    &+ \frac{6\eta^2}{n\rho} \mathbb{E} \|\mu B^t\|_F^2+\frac{\eta^2}{n} \mathbb{E} \|\delta^t-\Bar{\delta}^t\|_F^2\\
    \leq & \frac{(1-\rho/2)}{n} \mathbb{E} \|X^t -\Bar{X}^t\|_F^2 +{4 \eta^2 \sigma^2}+ \frac{6\eta^2}{n\rho} \mathbb{E} \|\mathbb{E}[G^t]-\nabla f(\Bar{x}^t)\|_F^2\\
    &+ \frac{6\eta^2\mu^2}{n\rho} \mathbb{E} \|B^t\|_F^2+ \underbrace{ \frac{\eta^2}{n\rho} \mathbb{E} \|\delta^t-\Bar{\delta}^t\|_F^2}_{Term (A)}\\
    \end{align*}

Term (A):

\begin{align*}
    A &= \ \frac{\eta^2}{n\rho} \E\left[\big\| {\delta}^{t} -\Bar{\delta}^{t} \big\|^2_{F}\right]
    \\
    & \leq \frac{2 \eta^2}{n \rho}\E\left[\big\| {\delta}^{t} \|^2_{F}\right] + \frac{2 \eta^2}{n \rho}\E\left[\big\|\Bar{\delta}^{t} \big\|^2_{F}\right]
    \\
    & \leq \frac{2 \eta^2}{n \rho} \sum_{j=1}^{n} \E\left[\big\|{\delta}^{t} \|^2\right] + \frac{ 2 \eta^2}{n \rho}\sum_{j=1}^{n}\E\left[\big\| \Bar{\delta}_{j}^{t} \big\|^2\right]
    \\
    & \leq \frac{2 \eta^2}{n \rho} \sum_{i=1}^{n}{D}^{2}_{t,i} + \frac{2 \eta^2} {n \rho} n \E\left[\big\| \frac{1}{n}\sum_{i=1}^{n}\delta^{t}_{i} \|^2\right]
    \\
    & \leq \frac{2 \eta^2}{n \rho} \sum_{i=1}^{n}{D}^{2}_{t,i} + \frac{2 \eta^2}{n \rho}\sum_{i=1}^{n}\E\left[\big\|\delta_{i}^{t} \big\|^2\right]
    \\
    & \leq \frac{2 \eta^2 }{n \rho} \sum_{i=1}^{n}{D}^{2}_{t,i} + \frac{2 \eta^2}{n \rho}\sum_{i=1}^{n}{D}^{2}_{t,i}
    \\
    & \leq \frac{4 \eta^2}{n \rho}\sum_{i=1}^{n}{D}^{2}_{t,i}
    \\
\end{align*}

Putting Term (A) back:
    \begin{align*}
    \frac{1}{n} \mathbb{E} \|X^{t+1} - \Bar{X}^{t+1}\|_F^2  \leq & \frac{(1-\rho/2)}{n} \mathbb{E} \|X^t -\Bar{X}^t\|_F^2 +{4 \eta^2 \sigma^2}+ \frac{6\eta^2\mu^2}{n\rho} \mathbb{E} \|B^t\|_F^2 + \frac{4\eta^2}{n\rho} \sum_{i=1}^{n}{D}^{2}_{t,i} \\
    &+ \frac{6\eta^2}{n\rho} \sum_{i=1}^n \mathbb{E} \|\nabla f_i(x_i^t) \pm \nabla f_i(\Bar{x}^t)-\nabla f(\Bar{x}^t)\|_F^2\\
    \overset{(e)}{\leq} & \frac{(1-\rho/2)}{n} \mathbb{E} \|X^t -\Bar{X}^t\|_F^2 +{4 \eta^2 \sigma^2}+ \frac{12 \eta^2 \zeta^2}{\rho}+\frac{6\eta^2\mu^2}{n\rho} \mathbb{E} \|B^t\|_F^2\\
    &+\frac{4\eta^2}{n\rho} \sum_{i=1}^{n}{D}^{2}_{t,i} + \frac{12\eta^2}{n\rho} \sum_{i=1}^n \mathbb{E} \|\nabla f_i(x_i^t) - \nabla f_i(\Bar{x}^t)\|_F^2\\
    \overset{(f)}{\leq} & \frac{(1-\rho/2)}{n} \mathbb{E} \|X^t -\Bar{X}^t\|_F^2 +{4 \eta^2 \sigma^2}+ \frac{12 \eta^2 \zeta^2}{\rho}+\frac{6\eta^2\mu^2}{n\rho} \mathbb{E} \|B^t\|_F^2\\
    &+\frac{4\eta^2}{n\rho} \sum_{i=1}^{n}{D}^{2}_{t,i} + \frac{12\eta^2 L^2}{n\rho} \sum_{i=1}^n \mathbb{E} \|x_i^t - \Bar{x}^t\|_F^2\\
    \le & \Big(\frac{1-\rho/2}{n}+ \frac{12\eta^2 L^2}{n\rho}\Big) \mathbb{E} \|X^t -\Bar{X}^t\|_F^2 +{4 \eta^2 \sigma^2}+ \frac{12 \eta^2 \zeta^2}{\rho}\\
    &+\frac{6\eta^2\mu^2}{n\rho} \mathbb{E} \|B^t\|_F^2 +\frac{4\eta^2}{n\rho} \sum_{i=1}^{n}{D}^{2}_{t,i} \\
     \overset{(g)}{\leq} & \frac{1-\rho/4}{n} \mathbb{E} \|X^t -\Bar{X}^t\|_F^2 +{4 \eta^2 \sigma^2}+ \frac{12 \eta^2 \zeta^2}{\rho}+\frac{6\eta^2\mu^2}{n\rho} \mathbb{E} \|B^t\|_F^2 \\
     &+ \frac{4\eta^2}{n\rho} \sum_{i=1}^{n}{D}^{2}_{t,i} \\
    \end{align*}
    \endgroup
    (a) Expanding using the formula $\mathbb{E}\|A + B + C\|_F^2 = \mathbb{E}\|A\|_F^2 + \mathbb{E}\|B\|_F^2 + \mathbb{E}\|C\|_F^2 + 2\mathbb{E}\langle A, B \rangle_F + 2\mathbb{E}\langle A, C \rangle_F + 2\mathbb{E}\langle B, C \rangle_F$. Cross-term of the expectation of random variables are zero
    (b) Results from the Assumption~\ref{as2}
    (c) follows from the fact that $\|a+b\|^2 \leq (1+\alpha) \|a\|^2 +(1+\frac{1}{\alpha})\|b\|^2 \hspace{2mm} \forall \alpha>0$ and let $\alpha=\frac{\rho}{2}$.
(d) From Assumption~\ref{as3}, $\|ZW-\Bar{Z}\|_F^2 \leq (1-\rho)\|Z-\Bar{Z}\|_F^2$ and $1+\frac{2}{\rho} \leq \frac{3}{\rho}$.
(e) Results from the Assumption~\ref{as4} 
(f) uses L-smoothness condition.
(g) Assumption  that $\eta \leq \frac{\rho}{7L}$
\end{proof}

The next step is to find an upper bound for the bias term $\mathbb{E} \|B^t\|_F^2$. 
\begin{lemma} 
\label{lemma4}
Given assumptions 1-3 and $\frac{\mu}{1-\mu} \leq \frac{\rho}{42}$, we have  \\
 $ \frac{6 \eta^2 \mu^2}{\rho n (1-\mu)} \mathbb{E} \|B^{t+1}\|_F^2 \leq  \Big(\frac{6 \eta^2 \mu^2}{n \rho (1-\mu)} - \frac{6 \eta^2 \mu^2}{n \rho} \Big) \mathbb{E}\| B^t\|_F^2  +\frac{\rho}{8n}\mathbb{E}\|X^{t}-\Bar{X}^t\|_F^2  +  \frac{\eta^2 \rho \sigma^2(1-\mu)}{8}  +  \frac{\eta^2 \rho   }{8 n} \sum_{i=1}^{n}{D}^{2}_{t,i} +\frac{ \eta^2 \rho \zeta^2}{ 8} $.
\end{lemma}

\begin{proof}
	starting from the update step \ref{eq:appendix_bias_matrix_form}
 \begingroup
\allowdisplaybreaks
	\begin{align*}
    B^{t+1}  =& -\frac{1}{\eta} [(2W-I) (X^{t+1}-X^{t}) +\eta G^{t}] \\
    =& -\frac{1}{\eta} [(2W-I) (WX^{t}-\eta G^t-\eta \mu B^t-\eta \delta^t-X^{t}) +\eta G^{t}] \\
    =& -\frac{1}{\eta} [W(2W-I)-I] X^{t}+2(W-I) G^t+ \mu (2W-I) B^t+ (2W-I) \delta^t] \\
    \end{align*}
    Now, 
    \begin{align*}
     \frac{1}{n} \mathbb{E}\|B^{t+1}\|_F^2 =& \frac{1}{n} \mathbb{E}\|-\frac{1}{\eta} (W(2W-I)-I) X^{t}+2(W-I) G^t+ \mu (2W-I) B^t + (2W-I) \delta^t\|_F^2 \\
     =& \frac{1}{n} \mathbb{E}\|-\frac{1}{\eta} (W(2W-I)-I) X^{t}+2(W-I)(G^t-\Bar{G}^t)+ \mu (2W-I) B^t + (2W-I) \delta^t\|_F^2 \\
     =& \frac{1}{n} \mathbb{E}\|-\frac{1}{\eta} (W(2W-I)-I) X^{t}+2(W-I)\mathbb{E}[G^t-\Bar{G}^t]+ \mu (2W-I) B^t \|_F^2 \\
     &+ \frac{1}{n}\mathbb{E}\|(2W-I) \delta^t\|_F^2 +\frac{1}{n} \mathbb{E}\|2(W-I)(G^t - \mathbb{E}[G^t]-(\Bar{G}^t-\mathbb{E}[\Bar{G}^t]))\|_F^2 \\
     \leq & \frac{1}{n} \mathbb{E}\|\frac{1}{\eta} (I-W(2W-I)) X^{t}+2(W-I)\mathbb{E}[G^t-\Bar{G}^t]+ \mu (2W-I) B^t \|_F^2 \\
     &+ \frac{1}{n}\mathbb{E}\| (2W-I) \delta^t\|_F^2 +{8\sigma^2}\\
     \leq & \frac{1}{n} \mathbb{E}\|\frac{1}{\eta} (I-W(2W-I)) X^{t}+2(W-I)\mathbb{E}[G^t-\Bar{G}^t]+ \mu (2W-I) B^t\|^2_F \\
     &+ \frac{1}{n} \mathbb{E}\|(2W-I) \delta^t\|_F^2 +{8\sigma^2}\\
      \overset{(a)}{\leq} & \frac{1}{n}\Big(1+\frac{1-\mu}{\mu}\Big) \mathbb{E}\|\mu (2W-I) B^t\|_F^2 +{8 \sigma^2} + \frac{1}{n} \mathbb{E}\|\delta^t\|_F^2\\
      &+\frac{1}{n}\Big(1+\frac{\mu}{1-\mu}\Big) \mathbb{E}\|\frac{1}{\eta} (I-W(2W-I)) X^{t}+2(W-I)\mathbb{E}[G^t-\Bar{G}^t]\|_F^2\\
      \leq & \frac{\mu}{n} \mathbb{E}\| B^t\|_F^2 + {8\sigma^2}+ \frac{1}{n} \sum_{i=1}^{n}{D}^{2}_{t,i} +\frac{2}{n(1-\mu)}\mathbb{E}\|\mathbb{E}[G^t-\Bar{G}^t]\|_F^2\\
      &+\frac{2}{n\eta^2(1-\mu)}\mathbb{E}\|(I-W(2W-I)) X^{t}\|_F^2\\
      = & \frac{(1-(1-\mu))}{n} \mathbb{E}\| B^t\|_F^2  +{8\sigma^2}+ \frac{1}{n}\sum_{i=1}^{n}{D}^{2}_{t,i}+ \frac{2}{n(1-\mu)}\mathbb{E}\|\mathbb{E}[G^t-\Bar{G}^t]\|_F^2\\
      &+\frac{2}{n\eta^2(1-\mu)}\mathbb{E}\|(2W+I)(I-W) X^{t}\|_F^2\\
       \overset{(b)}{\leq} & \frac{(1-(1-\mu))}{n} \mathbb{E}\| B^t\|_F^2 +\frac{8\sigma^2} {n}+ \frac{1}{n} \sum_{i=1}^{n}{D}^{2}_{t,i}+\frac{2}{n(1-\mu)}\mathbb{E}\|\mathbb{E}[G^t-\Bar{G}^t]\|_F^2\\
      &+\frac{18}{n\eta^2(1-\mu)}\mathbb{E}\|(I-W) X^{t}\|_F^2\\
      = & \frac{(1-(1-\mu))}{n} \mathbb{E}\| B^t\|_F^2  +{8\sigma^2} + \frac{1}{n} \sum_{i=1}^{n}{D}^{2}_{t,i}+\frac{2}{n(1-\mu)}\mathbb{E}\|\mathbb{E}[G^t-\Bar{G}^t]\|_F^2\\
      &+\frac{18}{n\eta^2(1-\mu)}\mathbb{E}\|(I-W)( X^{t}-\Bar{X}^t)\|_F^2\\
       \leq & \frac{(1-(1-\mu))}{n} \mathbb{E}\| B^t\|_F^2 +{8\sigma^2}  + \frac{1}{n} \sum_{i=1}^{n}{D}^{2}_{t,i}  +\frac{36}{n\eta^2(1-\mu)}\mathbb{E}\|X^{t}-\Bar{X}^t\|_F^2\\
      &+\frac{2}{n(1-\mu)}\mathbb{E}\|\mathbb{E}[G^t] \pm \nabla f(\Bar{x}^t)-\mathbb{E}[\Bar{G}^t]\|_F^2\\
      \leq & \frac{(1-(1-\mu))}{n} \mathbb{E}\| B^t\|_F^2 +{8\sigma^2} + \frac{1}{n} \sum_{i=1}^{n}{D}^{2}_{t,i} +\frac{36}{n\eta^2(1-\mu)}\mathbb{E}\|X^{t}-\Bar{X}^t\|_F^2\\
      &+\frac{8 \zeta^2}{1-\mu}+ \frac{4L^2}{n(1-\mu)}\mathbb{E}\|X^{t}-\Bar{X}^t\|_F^2\\
       = & \frac{(1-(1-\mu))}{n} \mathbb{E}\| B^t\|_F^2 +{8\sigma^2} + \frac{1}{n} \sum_{i=1}^{n}{D}^{2}_{t,i}  +\frac{4(9+\eta^2 L^2)}{n\eta^2(1-\mu)}\mathbb{E}\|X^{t}-\Bar{X}^t\|_F^2+\frac{8 \zeta^2}{1-\mu}\\
    \end{align*}
    Multiplying both sides with $\frac{6 \eta^2 \mu^2}{\rho (1-\mu)}$
    \begin{align*}
     \frac{6 \eta^2 \mu^2}{n \rho (1-\mu)}\mathbb{E}\|B^{t+1}\|_F^2 \leq & \Big(\frac{6 \eta^2 \mu^2}{n \rho (1-\mu)} - \frac{6\eta^2 \mu^2}{n \rho} \Big) \mathbb{E}\| B^t\|_F^2  +\frac{24(9+\eta^2 L^2)\mu^2}{n\rho(1-\mu)^2}\mathbb{E}\|X^{t}-\Bar{X}^)\|_F^2\\
     &+  \frac{48 \eta^2 \mu^2 \sigma^2}{\rho (1-\mu)} +  \frac{6 \eta^2 \mu^2 }{n \rho (1-\mu)}  \sum_{i=1}^{n}{D}^{2}_{t,i} +\frac{48 \eta^2 \mu^2 \zeta^2}{\rho (1-\mu)^2} \\
     \overset{(c)}{\leq} & \Big(\frac{6 \eta^2 \mu^2}{n \rho (1-\mu)} - \frac{6 \eta^2 \mu^2}{n \rho} \Big) \mathbb{E}\| B^t\|_F^2  +\frac{\rho}{8n}\mathbb{E}\|X^{t}-\Bar{X}^t\|_F^2\\
     &+  \frac{\eta^2 \rho \sigma^2(1-\mu)}{8}  +  \frac{\eta^2 \rho   }{8 n} \sum_{i=1}^{n}{D}^{2}_{t,i} +\frac{ \eta^2 \rho \zeta^2}{ 8}   \\
    \end{align*}
    \endgroup
    Note that  $W-I<I$,\hspace{2mm} $I-W<2I$ , \hspace{2mm} $(W-I)\Bar{X}^t=0$ and $(W-I)\Bar{G}^t=0$.
    (a) follows from the fact that $\|a+b\|^2 \leq (1+\alpha) \|a\|^2 +(1+\frac{1}{\alpha})\|b\|^2 \hspace{2mm} \forall \alpha>0$ and let $\alpha=\frac{1-\mu}{\mu}$.
    (b) uses the fact that $\|AB\|_F^2 \leq \sigma_{max}^2(A) \|B\|_F^2$ where $A=2W+I$, $B=(I-W)X^t$ and $\sigma_{max}^2(A)=9$.
    (c) uses the assumption $\frac{\mu}{1-\mu}\leq \frac{\rho}{42}$ and $\eta \leq \frac{\rho}{7L}$. This implies that $\frac{24(9+\eta^2 L^2) \mu^2}{\rho (1-\mu)^2}\leq \frac{\rho}{8}$, $\frac{48 \mu^2}{\rho (1-\mu)^2} \leq \frac{\rho}{8}$ and  $\frac{6 \mu^2 }{\rho (1-\mu)}  \leq \frac{\rho}{8} $
\end{proof}

\section{Main Results}

\begin{theorem}
\label{theorem_1}
(Convergence of \textit{FedNMUT} algorithm) Given Assumptions and let step size $\eta \leq \frac{\rho}{7L}$ and the scaling factor $\frac{\mu}{1-\mu} \leq \frac{\rho}{42}$.
For all $T \geq 1$, we have
\begin{equation}
    \begin{split}
          &\frac{1}{T} \sum_{t=0}^{T-1} \mathbb{E} \| \nabla f(\Bar{x}^t)\|^2 \leq \frac{2}{\eta T}(f(\Bar{x}^0  - f^* ) +  \frac{L \mu^2 \eta}{ n T}  \sum_{t=0}^{T-1} \mathbb{E} \|B^t\|_F^2 +  2 L^2 \eta^2 \sigma^2 \left[ \frac{16}{\rho n} + \frac{1-\mu}{2} +\frac{1}{2 n L \eta}\right]  \\
        &\hspace{4cm} + 2 L^2 \eta^2 \zeta^2 \left[\frac{48 }{\rho^2} + \frac{1}{2} \right]  +  \frac{2 L^2 \eta^2}{T} \sum_{t=0}^{T-1} \sum_{i=1}^{n}{D}^{2}_{t,i} \left[ \frac{16}{n\rho^2} + \frac{1}{2} +\frac{1}{2 n L \eta}  \right] 
    \end{split}
\end{equation}
where $f(\Bar{x}^0)-f^*$ is the sub-optimality gap, $\Bar{x}$ is the average/consensus model parameters.
\end{theorem}

\textbf{Proof:}

Recall Lemma~\ref{lemma3}
\begin{align*}
    \frac{1}{n} \mathbb{E} \|X^{t+1} - \Bar{X}^{t+1}\|_F^2  \leq &  \frac{1-\rho/4}{n} \mathbb{E} \|X^t -\Bar{X}^t\|_F^2 +{4 \eta^2 \sigma^2}+ \frac{12 \eta^2 \zeta^2}{\rho}+\frac{6\eta^2\mu^2}{n\rho} \mathbb{E} \|B^t\|_F^2 \\
     &+ \frac{4\eta^2}{n\rho}  \sum_{i=1}^{n}{D}^{2}_{t,i}. 
\end{align*}

and Lemma~\ref{lemma4}
\begin{align*}
         \frac{6 \eta^2 \mu^2}{n \rho (1-\mu)}\mathbb{E}\|B^{t+1}\|_F^2 &\leq \Big(\frac{6 \eta^2 \mu^2}{n \rho (1-\mu)} - \frac{6 \eta^2 \mu^2}{n \rho} \Big) \mathbb{E}\| B^t\|_F^2  +\frac{\rho}{8n}\mathbb{E}\|X^{t}-\Bar{X}^t\|_F^2\\
     &+  \frac{\eta^2 \rho \sigma^2(1-\mu)}{8}  +  \frac{\eta^2 \rho  }{8 n} \sum_{i=1}^{n}{D}^{2}_{t,i} +\frac{  \eta^2 \rho \zeta^2}{8}. \\
\end{align*}

Combining Lemma~\ref{lemma3} and Lemma~\ref{lemma4} and simplifying, we obtain
\begin{align*}
    \frac{1}{n} \mathbb{E} \|X^{t+1} - \Bar{X}^{t+1}\|_F^2 + \frac{6 \eta^2 \mu^2}{n \rho (1-\mu)}\mathbb{E}\|B^{t+1}\|_F^2 \leq & \frac{1-\rho/4}{n} \mathbb{E} \|X^t -\Bar{X}^t\|_F^2 +{4 \eta^2 \sigma^2}+ \frac{12 \eta^2 \zeta^2}{\rho} \\
    &+\frac{6\eta^2\mu^2}{n\rho} \mathbb{E} \|B^t\|_F^2 + \frac{4\eta^2}{n\rho} \sum_{i=1}^{n}{D}^{2}_{t,i} \\
    &+ \Big(\frac{6 \eta^2 \mu^2}{n \rho (1-\mu)} - \frac{6 \eta^2 \mu^2}{n \rho} \Big) \mathbb{E}\| B^t\|_F^2  +\frac{\rho}{8n}\mathbb{E}\|X^{t}-\Bar{X}^t\|_F^2 \\
    &+  \frac{\eta^2 \rho \sigma^2(1-\mu)}{8} +  \frac{\eta^2 \rho }{8 n}\sum_{i=1}^{n}{D}^{2}_{t,i} +\frac{ \eta^2 \rho \zeta^2}{8}
\end{align*}
Simplifying and multiplying the above equation by $\frac{4 L^2 \eta}{\rho}$ gives 
\begin{align*}
    \frac{4 L^2 \eta}{n \rho} \mathbb{E} \|X^{t+1} - \Bar{X}^{t+1}\|_F^2 + \frac{24 L^2 \eta^3 \mu^2}{n \rho^2 (1-\mu)}\mathbb{E}\|B^{t+1}\|_F^2 
    \end{align*}
    \begin{align*}
              \nonumber
        &\leq \left[ \frac{4 L^2 \eta}{\rho n} - \frac{L^2 \eta}{ n} \right] \mathbb{E} \|X^t -\Bar{X}^t\|_F^2 +\frac{16 L^2 \eta^3 \sigma^2}{\rho n}+ \frac{48 L^2 \eta^3 \zeta^2}{\rho^2} +\frac{24 L^2 \eta^3\mu^2}{n\rho^2} \mathbb{E} \|B^t\|_F^2 + \frac{16 L^2 \eta^3}{n\rho^2} D^{2,t} \\
                  \nonumber
        &\hspace{0cm} + \Big(\frac{24 L^2 \eta^3 \mu^2}{n \rho^2 (1-\mu)} - \frac{24 L^2 \eta^3 \mu^2}{n \rho^2} \Big) \mathbb{E}\| B^t\|_F^2  +\frac{L^2 \eta}{2n}\mathbb{E}\|X^{t}-\Bar{X}^t\|_F^2 +  \frac{L^2 \eta^3 \sigma^2(1-\mu)}{2} +  \frac{L^2 \eta^3 D^{t,2}}{2} +\frac{ L^2 \eta^3 \zeta^2}{2} \\
                  \nonumber
        \\
          &\leq \left[ \frac{4 L^2 \eta}{n \rho} - \frac{L^2 \eta}{2 n} \right] \mathbb{E} \|X^t -\Bar{X}^t\|_F^2  + \frac{24 L^2 \eta^3 \mu^2}{n \rho^2 (1-\mu)} \mathbb{E}\| B^t\|_F^2  + L^2 \eta^3 \sigma^2 \left[ \frac{16}{\rho n} + \frac{1-\mu}{2}\right] +  L^2 \eta^3 \zeta^2 \left[\frac{48 }{\rho^2} + \frac{1}{2} \right] \\
          \nonumber
        &\hspace{2cm} +  L^2 \eta^3 \sum_{i=1}^{n}{D}^{2}_{t,i} \left[ \frac{16}{n\rho^2} + \frac{1}{2}  \right]
        \label{eqn_Eqn8}
        \end{align*}

Let

\begin{equation}
    \label{eq:phi}
    \begin{split}
        \Phi^t = \frac{4 L^2 \eta}{n \rho} \mathbb{E} \|X^{t} - \Bar{X}^{t}\|_F^2 +\frac{24 L^2 \eta^3 \mu^2}{n \rho^2 (1-\mu)} \mathbb{E} \|B^{t}\|_F^2 + \mathbb{E}[f(\Bar{x}^t)-f^*]
    \end{split}
\end{equation} 
\\
\\
\\
Recall now Lemma~\ref{lemma2},
\begin{align*}
    \mathbb{E} f(\Bar{x}^{t+1}) &\leq \mathbb{E} f(\Bar{x}^{t}) + \frac{L \eta^2 \sigma^2}{2n}  -\frac{3\eta}{8} \ \mathbb{E} \|\frac{1}{n} \sum_{i=1}^n  \nabla f_i(x_i^t)\|^2 - \frac{\eta}{2} \ \mathbb{E} \| \nabla f(\Bar{x}^t)\|^2  \\
    &\hspace{2cm}+ \frac{L^2 \eta}{2n} \ \mathbb{E} \|X^t-\Bar{X}^t\|_F^2 + \frac{L \eta^2}{2n}D^{2,t} + \frac{L \mu^2 \eta^2}{2 n} \mathbb{E} \|B^t\|_F^2 
\end{align*}

Combining Lemma~\ref{lemma2} and  Equation \ref{eq:phi}, we have the following

\begin{equation*}
    \begin{split}
        \Phi^{t+1} &\leq \left[ \frac{4 L^2 \eta}{n \rho} - \frac{L^2 \eta}{2 n} \right] \mathbb{E} \|X^t -\Bar{X}^t\|_F^2  + \frac{24 L^2 \eta^3 \mu^2}{n \rho^2 (1-\mu)} \mathbb{E}\| B^t\|_F^2 \\
        &\hspace{2cm} + L^2 \eta^3 \sigma^2 \left[ \frac{16}{\rho n} + \frac{1-\mu}{2}\right] +  L^2 \eta^3 \zeta^2 \left[\frac{48 }{\rho^2} + \frac{1}{2} \right]  +  L^2 \eta^3 D^{t,2} \left[ \frac{16}{n\rho^2} + \frac{1}{2}  \right] \\
        &\hspace{2cm}+ \frac{L \eta^2 \sigma^2}{2n}  -\frac{3\eta}{8} \ \mathbb{E} \|\frac{1}{n} \sum_{i=1}^n  \nabla f_i(x_i^t)\|^2 - \frac{\eta}{2} \ \mathbb{E} \| \nabla f(\Bar{x}^t)\|^2  \\
    &\hspace{2cm}+ \frac{L^2 \eta}{2n} \ \mathbb{E} \|X^t-\Bar{X}^t\|_F^2 + \frac{L \eta^2}{2n}D^{2,t} + \frac{L \mu^2 \eta^2}{2 n} \mathbb{E} \|B^t\|_F^2 
    \\
    &\le \Phi^{t} +  L^2 \eta^3 \sigma^2 \left[ \frac{16}{\rho n} + \frac{1-\mu}{2} +\frac{1}{2 n L \eta}\right] +  L^2 \eta^3 \zeta^2 \left[\frac{48 }{\rho^2} + \frac{1}{2} \right]  \\
        &\hspace{1cm} +  L^2 \eta^3 D^{t,2} \left[ \frac{16}{n\rho^2} + \frac{1}{2} +\frac{1}{2 n L \eta}  \right] + \frac{L \mu^2 \eta^2}{2 n} \mathbb{E} \|B^t\|_F^2 \\
        &\hspace{1cm} -\frac{3\eta}{8} \ \mathbb{E} \|\frac{1}{n} \sum_{i=1}^n  \nabla f_i(x_i^t)\|^2 - \frac{\eta}{2} \ \mathbb{E} \| \nabla f(\Bar{x}^t)\|^2 
    \\
        \implies   &\frac{\eta}{2} \mathbb{E} \| \nabla f(\Bar{x}^t)\|^2 \leq (\Phi^{t}  - \Phi^{t+1} ) +  \frac{L \mu^2 \eta^2}{2 n} \mathbb{E} \|B^t\|_F^2 +  L^2 \eta^3 \sigma^2 \left[ \frac{16}{\rho n} + \frac{1-\mu}{2} +\frac{1}{2 n L \eta}\right] +  \\
        &\hspace{0.5cm}L^2 \eta^3 \zeta^2 \left[\frac{48 }{\rho^2} + \frac{1}{2} \right]  +  L^2 \eta^3 D^{t,2} \left[ \frac{16}{n\rho^2} + \frac{1}{2} +\frac{1}{2 n L \eta}  \right]  -\frac{3\eta}{8} \ \mathbb{E} \|\frac{1}{n} \sum_{i=1}^n  \nabla f_i(x_i^t)\|^2
    \end{split}
\end{equation*} 

Telescoping over iterations $t$ from $0$ to $T$,

\begin{align}
\label{eq:final}
          \nonumber
          &\frac{1}{T} \sum_{t=0}^{T-1} \mathbb{E} \| \nabla f(\Bar{x}^t)\|^2 \leq \frac{2}{\eta T}(f(\Bar{x}^0  - f^* ) +  \frac{L \mu^2 \eta}{ n T}  \sum_{t=0}^{T-1} \mathbb{E} \|B^t\|_F^2 +  2 L^2 \eta^2 \sigma^2 \left[ \frac{16}{\rho n} + \frac{1-\mu}{2} +\frac{1}{2 n L \eta}\right]  \\
        &\hspace{4cm} + 2 L^2 \eta^2 \zeta^2 \left[\frac{48 }{\rho^2} + \frac{1}{2} \right]  +  \frac{2 L^2 \eta^2}{T} \sum_{t=0}^{T-1} \sum_{i=1}^{n}{D}^{2}_{t,i} \left[ \frac{16}{n\rho^2} + \frac{1}{2} +\frac{1}{2 n L \eta}  \right] 
\end{align}


This concludes the proof of the Theorem~\ref{theorem_1}.

\subsection{Proof of Corollary}

Suppose that the step size $\eta=\mathcal{O}\Big(\sqrt{\frac{n}{T}}\Big)$, then for a sufficiently large $T$, we have 
$\Bar{B}^2= \frac{1}{T} \sum_{t = 0}^{T-1}\mathbb{E} \|B^t\|_F^2$ and $\Bar{D}^2= \frac{1}{nT} \sum_{t,i = 1,1}^{T,n}D^2_{t,i}$
\begin{equation*}
\begin{split}
    (i)  \hspace{2mm}& \eta \leq \min\Big\{\frac{1}{4L}, \frac{\rho}{7L}\Big\}\\
    (ii) \hspace{2mm}& \frac{\mu}{1-\mu} \leq \frac{\rho}{42}\\
    (iii) \hspace{2mm}& \frac{6 \mu^2 }{\rho (1-\mu)}  \leq \frac{\rho}{8} 
\end{split}
\end{equation*}
If the step size $\eta$ is $\mathcal{O}(\sqrt{\frac{n}{T}})$, then we have the following order of convergence for each term in Theorem~\ref{thm-all}:

\[\frac{2}{\eta T}(f(\Bar{x}^0  - f^* )=\mathcal{O}\left[\frac{1}{\sqrt{nT}}\right]=\mathcal{O}\left[\frac{1}{\sqrt{T}}\right]\]
\begin{align*}
    \frac{L \mu^2 \eta}{ n T}  \sum_{t=0}^{T-1} \mathbb{E} \|B^t\|_F^2 &= \mathcal{O}\left[\frac{1}{ \sqrt{n T}}\right]\Bar{B}^2= \mathcal{O}\left[\frac{1}{ \sqrt{ T}}\Bar{B}^2\right]
    \\
    2 L^2 \eta^2 \sigma^2 \left[ \frac{16}{\rho n} + \frac{1-\mu}{2} +\frac{1}{2 n L \eta}\right] &=   \frac{2 L^2 \eta^2 \sigma^2}{\rho n} +  L^2 \eta^2 \sigma^2(1-\mu) +\frac{L \eta \sigma^2}{n} \\
    &= \mathcal{O}\left[\frac{1}{\sqrt{ T}} \sigma^2\right]
    \\
    2 L^2 \eta^2 \zeta^2 \left[\frac{48 }{\rho^2} + \frac{1}{2} \right] &=   \frac{96 L^2 \eta^2 \zeta^2 }{\rho^2} + L^2 \eta^2 \zeta^2 \\
    &= \mathcal{O}\left[\frac{1}{T}  \zeta^2\right]
    \\
    \frac{2 L^2 \eta^2}{T} \sum_{t=0}^{T-1} \sum_{i=1}^{n}{D}^{2}_{t,i} \left[ \frac{16}{n\rho^2} + \frac{1}{2} +\frac{1}{2 n L \eta}    \right] 
     &= \frac{32 L^2 \eta^2 \Bar{D}^2 n T}{n\rho^2 T} +  \frac{L^2 \eta^2 \Bar{D}^2 n T}{T} +\frac{L \eta \Bar{D}^2 n T}{ n T }  \\
    &= \mathcal{O}\left[\frac{1}{\sqrt{T}} \Bar{D}^2 \right]
\end{align*}

The overall convergence rate is
\begin{equation*}
     \frac{1}{T} \sum_{t=0}^{T-1} \mathbb{E} \| \nabla f(\Bar{x}^t)\|^2 \leq \mathcal{O}\left[\frac{1}{ \sqrt{ T}}\Bar{B}^2+\frac{1}{\sqrt{ T}} \sigma^2+\frac{1}{T}  \zeta^2+\frac{1}{\sqrt{T}} \Bar{D}^2 \right], 
\end{equation*}

Therefore, at large $T$, the convergence rate of \textit{FedNMUT} is\textbf{ $\mathcal{O}(\frac{1}{\sqrt{T}})$}.

\end{document}